\documentclass[10pt,journal,compsoc]{IEEEtran}
\ifCLASSOPTIONcompsoc
  \usepackage[nocompress]{cite}
\else
  \usepackage{cite}
\fi

\ifCLASSINFOpdf
\else
\fi
\usepackage{amsmath,mathrsfs,amssymb}
\usepackage{hyperref}
\usepackage{xcolor}

\usepackage{booktabs}       
\usepackage{longtable}
\usepackage{threeparttable}
\usepackage{multirow}
\usepackage{multicol}
\usepackage{graphicx}
\usepackage{bm}

\usepackage{amsthm}
\usepackage{dsfont}
\usepackage{lipsum} 
\newtheorem{theorem}{Theorem}
\newtheorem{lemma}{Lemma}
\newtheorem{definition}{Definition}

\newtheorem{prop}{Proposition}

\def\m{\mathfrak{m}}
\def\real{\mathbb{R}}

\def\p{\mathbf{p}}

\def\tf{\tilde{f}}

\def\q{\mathbf{q}}
\def\z{\mathbf{z}}

\def\m{\mathfrak{m}}
\def\real{\mathbb{R}}

\def\p{\mathbf{p}}

\def\b{\textbf{b}}

\def\tf{\tilde{f}}

\usepackage{cuted}

\def\x{\mathbf{x}}
\def\n{\mathbf{n}}
\def\c{\mathbf{c}}
\def\y{\mathbf{y}}
\def\W{\mathbf{W}}
\def\m{\mathbf{m}}
\def\V{\mathbf{V}}
\def\M{\mathcal{M}}
\def\B{\mathcal{B}}
\def\R{\mathbb{R}}

\def\w{\textbf{w}}

\newcommand{\ff}[1] {\textcolor{black}{#1}}

\begin{document}
%
\title{One Neuron Saved Is One Neuron Earned: On Parametric Efficiency of Quadratic Networks}
%
%
%
%

\author{Feng-Lei Fan, \textit{Member, IEEE}, Hang-Cheng Dong, Zhongming Wu, Lecheng Ruan, Tieyong Zeng, Yiming Cui, Jing-Xiao Liao 
\IEEEcompsocitemizethanks{
\IEEEcompsocthanksitem Feng-Lei Fan, Zhong-Ming Wu, and Tieyong Zeng are with Department of Mathematics, Chinese University of Hong Kong, Shatin, Hong Kong. 
\IEEEcompsocthanksitem 
Jing-Xiao Liao and Hang-Cheng Dong are with School of Instrumentation Science and Engineering, Harbin Institute of Technology, Harbin, Heilongjiang Province 150001, China. 
\IEEEcompsocthanksitem Yiming Cui is with Department of Electrical and Computer Engineering, University of Florida, Gainesville, FL 32611, USA.
\IEEEcompsocthanksitem Lecheng Ruan is with the National Key Laboratory of General Artificial Intelligence, Beijing Institute for General Artificial Intelligence (BIGAI), Beijing, 100080, China.}
}

\markboth{Journal of \LaTeX\ Class Files,~Vol.~14, No.~8, August~2023}%
{Shell \MakeLowercase{\textit{et al.}}: Bare Demo of IEEEtran.cls for Computer Society Journals}
%



\IEEEtitleabstractindextext{%
\begin{abstract}
Inspired by neuronal diversity in the biological neural system, a plethora of studies proposed to design novel types of artificial neurons and introduce neuronal diversity into artificial neural networks. 
Recently proposed quadratic neuron, which replaces the inner-product operation in conventional neurons with a quadratic one, have achieved great success in many essential tasks. Despite the promising results of quadratic neurons, there is still an unresolved issue: \textit{Is the superior performance of quadratic networks simply due to the increased parameters or due to the intrinsic expressive capability?} Without clarifying this issue, the performance of quadratic networks is always suspicious. Additionally, resolving this issue is reduced to finding killer applications of quadratic networks. In this paper, with theoretical and empirical studies, we show that quadratic networks enjoy parametric efficiency, thereby confirming that the superior performance of quadratic networks is due to the intrinsic expressive capability. This intrinsic expressive ability comes from that quadratic neurons can easily represent nonlinear interaction, while it is hard for conventional neurons. Theoretically, we derive the approximation efficiency of the quadratic network over conventional ones in terms of real space and manifolds. Moreover, from the perspective of the Barron space, we demonstrate that there exists a functional space whose functions can be approximated by quadratic networks in a dimension-free error, but the approximation error of conventional networks is dependent on dimensions. Empirically, experimental results on synthetic data, classic benchmarks, and real-world applications show that quadratic models broadly enjoy parametric efficiency, and the gain of efficiency depends on the task.
\ff{We have shared our code in \url{https://github.com/asdvfghg/quadratic_efficiency}.}
\end{abstract}

}

\maketitle

\IEEEdisplaynontitleabstractindextext

%
\IEEEpeerreviewmaketitle

\IEEEraisesectionheading{\section{Introduction}\label{sec:introduction}}

%
%
%
%
\IEEEPARstart{T}{he} brain is the most intelligent system we have ever known so far. Historically, neuroscience has been a great source of inspiration for the research of artificial networks, \textit{e.g.}, the modeling of artificial neurons \cite{mcculloch1943logical} and the invention of neocognitron \cite{fukushima1982neocognitron}, which is the pioneering work of convolutional models. \ff{Recently, the concept of "NeuroAI" \cite{zador2022toward} was coined that although artificial networks
would usually follow their own paths after drawing inspiration from neuroscience, a large amount of neuroscience knowledge can greatly support the development of deep learning.
In this light,} as we know, our brain is composed of numerous morphologically and functionally different neurons \cite{peng2021morphological}. It is no exaggeration to say that neuronal diversity is an essential factor for all kinds of biological intelligent behaviors. However, in the past decade, most deep learning research is developing powerful network models by designing novel architectures, \textit{i.e.}, shortcut design \cite{he2016deep, huang2017densely, fan2021sparse}. These successful networks are almost exclusively built on the same type of neurons, which are composed of an inner product and nonlinear activation, and thus lack neuronal diversity. For convenience, we refer to this type of neurons as conventional neurons, and a network made of these neurons as a conventional network. Our curiosity is that since an artificial network is a miniature of the biological neural network, the role and potential gains of endowing an artificial network with neuronal diversity should be carefully examined.

Encouraged by the idea of neuronal diversity, our group previously prototyped a new type of neurons known as \textit{quadratic neurons}, which substitute the inner product in a conventional neuron with a simplified quadratic function. Hereafter, we denote the network built on quadratic neurons as the \textit{quadratic network}. Fan \textit{et al.} demonstrated that a single quadratic neuron can execute the XOR logic operation, while a conventional neuron cannot \cite{fan2018new}. Furthermore, the superior expressivity of quadratic networks was shown by the spline theory \cite{fan2021expressivity}---with ReLU as activation functions, a conventional network is a piecewise linear function, while a quadratic one is a piecewise polynomial. According to the spline theory, a piecewise polynomial spline typically has a lower approximation error than that of a piecewise linear one, thereby justifying the superior expressivity of quadratic networks. Concurrently, the ReLinear algorithm was proposed to address the training instability issue of quadratic networks \cite{fan2021expressivity}. This algorithm initializes the quadratic network as the conventional network and controls the learning rate of quadratic terms. Subsequently, quadratic networks were extensively applied to solving real-world problems, \textit{e.g.}, low-dose CT denoising \cite{fan2019quadratic} and bearing fault diagnosis \cite{liao2022attention}. This series of quadratic neuron studies aim to develop deep learning based on new neurons and validate the potential of introducing neuronal diversity into deep learning.

Although highly non-trivial success was achieved by quadratic networks, the parametric efficiency issue remains unresolved satisfactorily. \textit{Is the superior performance of quadratic networks simply due to the increased parameters or due to the intrinsic expressive capability?} This question is extremely important to the development of quadratic networks, as well as other types of polynomial networks. Without clarification, the performance of quadratic networks is always suspicious. Additionally, by reasonably assuming that quadratic networks cannot achieve parametric efficiency for all tasks compared to conventional ones, answering this question is actually reduced to finding killer applications of quadratic networks, \textit{i.e.}, what tasks a quadratic network fits most such that it can achieve higher performance with far fewer parameters. Our earlier theoretical work proved that there exists a class of functions that can be approximated by a polynomial number of quadratic neurons, but have to be approximated by an exponential number of conventional neurons to achieve the same level of error \cite{fan2020universal}. But such a result is neither general nor pragmatic, \textit{i.e.}, it just constructs a very special class of functions theoretically, and it sheds no light on the parametric efficiency pragmatically.

Given the significance of parametric efficiency, we systematically and thoroughly elaborate on it from both theoretical and empirical perspectives. \ff{We show that quadratic networks enjoy parametric efficiency, thereby confirming that the superior performance of quadratic networks is due to the intrinsic expressive capability instead of the increased parameters. This intrinsic expressive ability comes from that quadratic neurons can easily represent nonlinear interaction, while it is hard for conventional neurons.} Theoretically, 1) Yarotsky earlier constructed the conventional ReLU networks to approximate functions in Sobolev space \cite{Yarotsky2017Error} and established the tight error bounds for the structure and parameters of networks. Chen \textit{et al.} modified such an analysis scheme to approximate functions over low dimensional manifolds \cite{chen2019efficient}. Both schemes achieve the optimal error rate \cite{devore1989optimal}. In analogy, by the same schemes, we construct the quadratic ReLU networks to approximate functions over Sobolov space and a manifold, respectively. Due to the interaction terms embedded in the quadratic neuron, \ff{it is easy for quadratic neurons to represent nonlinear interaction (without the need of stacking parameters), while it is hard for conventional neurons to represent such a nonlinear interplay}. Thus, it is shown in Theorems \ref{real_space} and \ref{manifold_approximation} that at the same level of the approximation error, the constructed quadratic networks use an order of magnitude fewer parameters than the conventional counterpart. 2) From the perspective of Barron space, we demonstrate that there exists a functional space whose functions can be approximated by quadratic networks without the curse of dimensions, while the approximation error of conventional networks is dependent on dimensions (Theorem \ref{efficiency_barron}). \ff{Such a difference is because the added quadratic operations expand the Barron space instead of adding more parameters.} Moreover, the opposite case does not hold because the conventional network can be regarded as a degenerated case of the quadratic.

Empirically, systematic experiments over synthetic data, classic benchmarks, and real-world applications verify the parametric efficiency of quadratic networks relative to other competitive baselines. Experiments also confirm that in cell and leaf segmentation tasks whose objects have significant circular features, quadratic networks are of high efficiency. \ff{Therefore, we think that quadratic networks can be the priority model for this kind of tasks.} In addition, although our highlight is the parametric efficiency, the number of FLOPs is also saved in quadratic models in some tasks, thereby exhibiting computational efficiency. To summarize, our contributions are twofold:

\begin{itemize}
    \item By approximating functions over Sobolov space, manifolds, and Barron space, we prove three highly non-trivial theorems, suggesting that quadratic networks should enjoy parametric efficiency in a broad sense \ff{and the superior performance of quadratic networks is due to the intrinsic nonlinear representation ability instead of the increased parameters}. 
    
    \item With systematic experiments, we confirm that the quadratic network is more efficient than conventional networks in many important tasks. In addition, we find that the quadratic networks can be the priority for tasks whose circular features are significant such as cell and leaf segmentation tasks.
\end{itemize}

\section{Related Works}
\textbf{Polynomial neurons and quadratic neurons.} In the 1970s, Ivakhnenko proposed the so-called Group Method of Data Handling (GMDH) \cite{ivakhnenko1971polynomial} which uses a high-order polynomial as a feature extractor. Later, Shin \textit{et al.} proposed a higher-order network called the pi-sigma network \cite{shin1991pi}:
\begin{equation}
    y_i = \sigma(\prod_{j}({\sum_{k}{w_{kji}{x}_k+\theta_{ji}}})),
\end{equation}
where $y_i$ is the $i$-th output, $\sigma(\cdot)$ is a nonlinear activation function, $w_{kji}$ and $\theta_{ji}$ are the weights of the $j$-th unit with respect to the input ${x}_k$. The pi-sigma network is essentially a special polynomial network.
To enhance the parametric efficiency of polynomial networks, Milenkovic \textit{et al.} proposed the second-order neural networks by directly removing cubic and higher-order terms \cite{milenkovic1996annealing}. 

These days, higher-order neurons were revisited and introduced into the fully-connected and convolutional networks. In \cite{chrysos2021deep} and \cite{liu2021dendrite}, the complexity of the higher-order units was greatly reduced via tensor decomposition and factor sharing. More studies are concerned with quadratic neurons, which directly abandon higher-order terms and use matrix decomposition to achieve parametric efficiency. We summarize the recently-proposed quadratic neurons in Table \ref{tab:neurons}. It can be seen that the complexity of neurons in \cite{zoumpourlis2017non, micikevicius2017mixed,jiang2020nonlinear,mantini2021cqnn} is of $\mathcal{O}(n^2)$, where $n$ is the dimension of the input, which is still too high to prototype deep networks. Also, neurons in \cite{goyal2020improved,bu2021quadratic,xu2022quadralib} are the special cases of our proposed neuron \cite{fan2018new}.

\begin{table}[htbp]
\centering
\caption{A summary of the recently-proposed quadratic neurons. $\sigma(\cdot)$ denotes the nonlinear activation function. $\odot$ denotes Hadamard product. Here, $\W \in \mathbb{R}^{n\times n}$, $\w_i \in \mathbb{R}^{n\times 1}$, and we omit the bias terms in these neurons.}
\scalebox{0.9}{
\begin{tabular}{l|l}
\hline
Authors           & Formulations          \\ \hline
Zoumpourlis \textit{et al.}(2017) \cite{zoumpourlis2017non, micikevicius2017mixed} & $\mathbf{y}=\sigma(\x^{\top}\W\x+\w^\top\x)$               \\ \hline
Jiang \textit{et al.}(2019) \cite{jiang2020nonlinear}      & \multirow{2}{*}{$\y=\sigma(\x^{\top}\W\x$)} \\ \cline{1-1}
Mantini\&Shah(2021) \cite{mantini2021cqnn}     &                        \\ \hline
Goyal \textit{et al.}(2020) \cite{goyal2020improved}    & $\y=\sigma(\w^\top(\x\odot\x)$               \\ \hline
Bu\&Karpatne(2021) \cite{bu2021quadratic}       & $\y=\sigma((\w_1^\top\x)(\w_2^\top\x))$             \\ \hline
Xu \textit{et al.}(2022) \cite{xu2022quadralib} & $\y=\sigma((\w_1^\top\x) (\w_2^\top\x)+\w_3^\top\x)$ \\
\hline
Fan \textit{et al.}(2018)   \cite{fan2018new}      & $\y=\sigma((\w_1^\top\x)(\w_2^\top\x)+\w_3^\top(\x\odot\x))$        \\ \hline
\end{tabular}}
\label{tab:neurons}
\end{table}

\textbf{Advances in polynomial networks.} In the theoretical
aspect, on the one hand, the expressive capability of polynomial networks was discussed. With the help of algebraic geometry, Kileel \textit{et al.} used the dimension of algebraic variety to characterize the expressive ability of polynomial networks \cite{kileel2019expressive}. Liao \textit{et al.} proved that there exists a function that a heterogeneous network made up of conventional and quadratic neurons can approximate with a polynomial number of neurons but a homogeneous conventional or quadratic network needs an exponential number of neurons to achieve the same level of error \cite{liao2022heterogeneous}. On the other hand, the robustness verification of polynomial networks was addressed. The existing verification algorithms on conventional networks
based on classical branch and bound techniques cannot be directly extended to polynomial networks. Rocamora \textit{et al.} \cite{rocamora2022sound} leveraged the twice differentiability of polynomial networks to build a lower bounding method for branch and bound techniques to realize a complete verification of robustness of polynomial networks.

In addition, more and more research attempted to use polynomial networks in real-world applications. Fan \textit{et al.} proposed a quadratic autoencoder for low-dose CT denoising \cite{fan2019quadratic}. Nguyen \textit{et al.} employed quadratic networks to predict the compressive strength of foamed concrete \cite{nguyen2019deep}. Bu \textit{et al.} \cite{bu2021quadratic}
explored how to use a quadratic network to solve forward and inverse problems in partial differential equations. Qi and Wang showed the superiority of quadratic networks for classifying Gaussian mixture data \cite{qi2022superiority}. Xu \textit{et al.} built a quadratic network library (\href{https://github.com/zarekxu/QuadraLib}{QuadraLib}) featuring algorithm acceleration to facilitate the application studies of quadratic networks \cite{xu2022quadralib}. Liao \textit{et al.} applied the quadratic network to the bearing fault diagnosis on one-dimensional vibration signals, and discovered the attention mechanism inherent in a quadratic neuron (\textit{qttention}) \cite{liao2022attention}. Chrysos \textit{et al.} studied the architectures of polynomial networks and achieved the state-of-the-art performance on image classification and generation \cite{chrysos2022augmenting}. However, in spite of success in a myriad of tasks, it remains unclear what kind of applications quadratic models fit most. 

\section{Theory of Efficiency}

\subsection{Preliminaries}

Mathematically, given the $d$-dimensional input $\x \in \mathbb{R}^{d\times 1}$, the conventional neuron computes the inner product between the input and the weight vector $\w$ in the neuron, followed by a nonlinear activation function $\sigma(\cdot)$: 
\begin{equation}
    y=\sigma(\w^\top\x+b).
\end{equation}
In contrast, quadratic neurons replace the inner product with a simplified quadratic function:
\begin{equation} 
y=\sigma((\w_1^{\top}\x+b_1)(\w_2^{\top}\x+b_2)+\w_3^{\top}(\x\odot\x)+b_3),
\label{qu}
\end{equation}
where $\w_1, \w_2, \w_3$ are weight vectors, $b_1, b_2, b_3$ are biases, and $\odot$ is Hadamard product. Unless otherwise specified, the activation function throughout this draft is the ReLU activation, which is widely used in neural networks. 


\subsection{Efficiency: Functions on the Sobolov Space}
\label{subsec:euclidean}

\begin{definition}[Sobolov space]
Let $\mathcal{W}^{n,\infty}([0,1]^d)$ be the Sobolev space which comprises of functions on $[0,1]^d$ lying in $L^{\infty}$ along with their weak derivatives up to order $n$. The norm of a function $f$ in $\mathcal{W}^{n,\infty}([0,1]^d)$ is 
\begin{equation}
\|f\|_{\mathcal{W}^{n, \infty}\left([0,1]^{d}\right)}=\max _{\mathbf{n}:|\mathbf{n}| \leq n} \mathrm{ess}~\underset{\mathbf{x} \in[0,1]^{d}}{\operatorname{sup}}\left|D^{\mathbf{n}} f(\mathbf{x})\right|,
\end{equation}
where $\mathbf{n}=(n_1,n_2,\dots,n_d) \in \{0,1,\dots,n\}^d$, $|\mathbf{n}|=n_1+n_2+\dots+n_d \leq n$, and $D^{\mathbf{n}}f$ is the $\mathbf{n}$-order weak derivative. Essentially, the space $\mathcal{W}^{n,\infty}([0,1]^d)$ is $C^{n-1}([0,1]^d)$ whose functions' up to order $n$ derivatives are Lipschitz continuous.
\label{def:f}
\end{definition}

\begin{definition}[Function space]
Let $F_{n,d}$ be a set of functions lying in the unit ball in $\mathcal{W}^{n,\infty}([0,1]^d)$, we have
\begin{equation}
    F_{n,d} = \{f \in \mathcal{W}^{n,\infty}([0,1]^d) : \|f\|_{\mathcal{W}^{n,\infty}([0,1]^d)} \leq 1\}.
\label{eq2:F}
\end{equation}
\label{def:F}
\end{definition}

Note that the space $\mathcal{W}^{n,\infty}([0,1]^d)$ is sufficiently general because functions used in most real-world learning tasks belong to $\mathcal{W}^{n,\infty}([0,1]^d)$. For simplicity, we use functions in $ F_{n, d}$ instead of $\mathcal{W}^{n,\infty}([0,1]^d)$ to verify the approximation efficiency of quadratic networks. 

\begin{lemma}
There exists a one-hidden-layer quadratic ReLU network that can implement a mapping $\tilde{Q}: \real^2 \to \real$, satisfying that i) $\tilde{Q}(x,y)=xy$; ii) $\tilde{Q}(x,y)$ has $2$ quadratic neurons and accordingly $4$ parameters.
\begin{proof}
First, given the input $(x,y)$, by appropriately setting parameters in Eq. \eqref{qu}, we have 
\begin{equation}
\begin{aligned}
& \tilde{f}(x, y) \\
=&(w_{11}x+w_{12}y+b_1)(w_{21}x+w_{22}y+b_2)+w_{31}x^2 \\
&+w_{32}y^2 + b_3\\
=&xy,\\
\end{aligned}
\label{eq2:quadraticf}
\end{equation}
where $w_{11}=1, w_{12}=0, b_1=0; w_{21}=0, w_{22}=1, b_2=0; w_{31}=0, w_{32}=0, b_3=0$.
Then, with the ReLU function, summating two parallel quadratic neurons with the opposite phases can constitute a product operation:
\begin{equation}
    xy=\sigma(\tilde{f}(x, y)) + \sigma(-\tilde{f}(x, y))=\tilde{Q}(x,y).
\end{equation}
$\tilde{Q}(x, y)$ is a one-hidden-layer quadratic network with just $2$ neurons and $4$ parameters. 


\end{proof}
\label{prop:xy}
\end{lemma}

\begin{theorem}
There exists a quadratic ReLU network, for any $d, n >0$ and $\epsilon \in (0,1)$, we have

1. a quadratic ReLU network can approximate any function from $F_{n,d}$ with an error bound $\epsilon$ in the sense of $L^{\infty}$;

2. this quadratic ReLU network has at most $4d^n(N+1)^d(d+n-1)$ parameters.
\label{real_space}
\end{theorem}

\textbf{The sketch of proof.}
First, given $f \in F_{n,d}$, we divide the input space into $(N+1)^d$ hyper-grids by a partition function $\phi_{\mathbf{m}}$, and use the sum-product combination of local Taylor polynomials $f_1$ with respect to each grid to approximate $f$. Then, we construct a quadratic ReLU network $\hat{f}$ to express $f_1$ without any error. Thus, $\hat{f}$ can also approximate $f$. Lastly, we count the number of parameters used in $\hat{f}$.

\begin{proof} 

\textbf{Step 1: Construct $f_1$ to approximate $f$.} First, we need the following function to partition the input space into hyper-grids: 
\begin{equation}
    \phi_{\mathbf{m}}(\mathbf{x})=\prod_{k=1}^{d}\psi \left ( 3N(x_k-\frac{m_k}{N} ) \right ),
\label{eq2:phimx}
\end{equation}
where
\begin{equation}
\nonumber
\psi(x)=\begin{cases}
 1, &   |x|<1 \\
 0, &  |x|>2 \\
 2-|x|, & 1 \leq  |x| \leq 2.
\end{cases}
\end{equation}
The characteristics of $\phi_{\mathbf{m}}(\mathbf{x})$ are that i) the maximum value of $\phi_{\mathbf{m}}(\mathbf{x})$ is $1$:
\begin{equation}
    \left \| \psi  \right \| _\infty =1 \ \operatorname{and}  \ \left \| \phi_\mathbf{m}  \right \| _\infty = 1 \ ,\forall \mathbf{m} ; 
\label{eq2:constraint}
\end{equation}
ii) a collection of $\phi_{\mathbf{m}}{(\mathbf{x})}$ on the domain $[0,1]^d$ forms a unity:
\begin{equation}
    \nonumber
    \sum_{\mathbf{m}}{\phi_{\mathbf{m}}(\mathbf{x})} = 1, \mathbf{x} \in [0,1]^d,
\end{equation}
where $\mathbf{m}=(m_1,\dots,m_d)\in\{0,1,\dots,N\}^d$; iii) $\phi_{\mathbf{m}}{(\mathbf{x})}$ is a local function:
\begin{equation}
    \operatorname{supp} \ \phi _{\mathbf{m} }\subset \left \{\mathbf{x}:|x_k - \frac{m_k}{N} |<\frac{1}{N} \ ,\forall k  \right \} .
\label{eq2:support}
\end{equation}
$\forall~\mathbf{m} \in \{0,\dots,N\}^d$, let $P_{\mathbf{m}}(\mathbf{x})$ be the $(n-1)$-degree Taylor polynomial of the function $f$ at $\mathbf{x}=\frac{\mathbf{m}}{N}$, we have
\begin{equation}
    P_{\mathbf{m} }(\mathbf{x})=\sum_{\mathbf{n}: |\mathbf{n}|<n}{\left.\frac{D^\mathbf{n}f}{\mathbf{n}!}\right|_{\mathbf{x} =\frac{\mathbf{m}}{N}}
\left ( \mathbf{x}-\frac{\mathbf{m} }{N}  \right )^\mathbf{n} }, 
\label{eq2:taylor}
\end{equation}
where $\mathbf{n}!=\textstyle \prod_{k=1}^{d}n_k!$ and $\left ( \mathbf{x}-\frac{\mathbf{m} }{N}  \right )^\mathbf{n}=\textstyle \prod_{k=1}^{d}\left ( {x}-\frac{{m} }{N}  \right )^{n}$. 

Define $f_1$ as
\begin{equation}
    f_1 = \sum_{\mathbf{m} \in \{0,\dots,N\}^d}{\phi_{\mathbf{m}}P_{\mathbf{m}}}.
\label{eq2:f1}
\end{equation}
Due to the cutoff effect of $\phi_{\mathbf{m}}$, $f_1$ is the sum of Taylor expansions of $f$ over $(N+1)^d$ hyper-grids. $f_1$ can approximate $f$ arbitrarily well as long as $N$ is sufficiently large. Mathematically, we have
\begin{equation}
    \begin{aligned}
&\left|f(\mathbf{x})-f_{1}(\mathbf{x})\right| \\
=&\left|\sum_{\mathbf{m}} \phi_{\mathbf{m}}(\mathbf{x})\left(f(\mathbf{x})-P_{\mathbf{m}}(\mathbf{x})\right)\right| \\
\overset{(1)}{\leq} & \sum_{\mathbf{m}:\left|x_{k}-\frac{m_{k}}{N}\right|<\frac{1}{N} \forall k}\left|f(\mathbf{x})-P_{\mathbf{m}}(\mathbf{x})\right| \\
 \overset{(2)}{\leq} & 2^{d} \max _{\mathbf{m}:\left|x_{k}-\frac{m_{k}}{N}\right|<\frac{1}{N} \forall k}\left|f(\mathbf{x})-P_{\mathbf{m}}(\mathbf{x})\right| \\
\overset{(3)}{\leq} & \frac{2^{d} d^{n}}{n !}\left(\frac{1}{N}\right)^{n} \max _{\mathbf{n}:|\mathbf{n}|=n} \operatorname{ess}~\underset{\mathbf{x} \in[0,1]^{d}}{\operatorname{sup}}\left|D^{\mathbf{n}} f(\mathbf{x})\right| \\
 \overset{(4)}{\leq} & \frac{2^{d} d^{n}}{n !}\left(\frac{1}{N}\right)^{n} .
\end{aligned}
\end{equation}
(1) follows from Eqs. (\ref{eq2:constraint}) and (\ref{eq2:support}); (2) follows from the fact that an arbitrary $\mathbf{x}$ belongs to the support of at most $2^d$ functions $\phi_{\mathbf{m}}$; (3) follows from a standard bound for the Taylor remainder; (4) follows from $\|f\|_{\mathcal{W}^{n,\infty}([0,1]^d)} \leq 1$, which is our definition. Hence, let
\begin{equation}
   N=\left \lceil \left ( \frac{n!}{2^dd^n} \frac{\epsilon }{2}  \right )^{-1/n} \right \rceil,
\end{equation}
where $\lceil \cdot \rceil$ is the ceiling function, we have
$
    \|f-f_1\|_\infty \leq \frac{\epsilon}{2}.
$

\begin{figure}[htbp]
    \centering
    \includegraphics[width=0.95\linewidth]{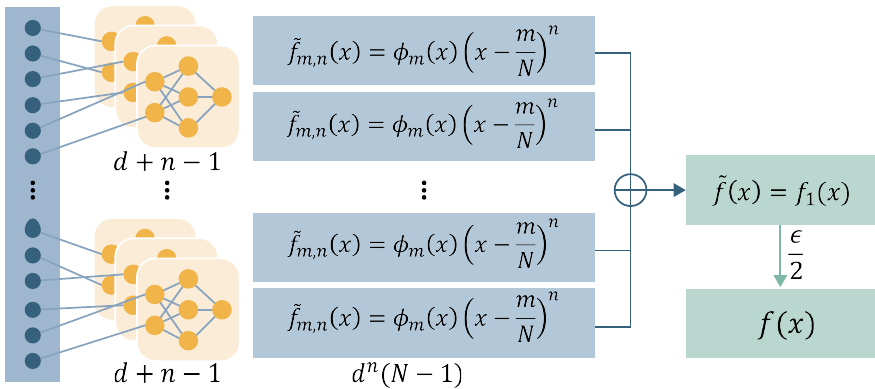}
    \caption{The architecture of using a quadratic network to approximate an arbitrary function $f$ in $F_{n, d}$.}
    \label{fig:proofsekcth}
\end{figure}

\textbf{Step 2: Construct a quadratic network $\tilde{f}$ to approximate $f_1$.} 
Let $a_{\mathbf{m},\mathbf{n}}=\left.\frac{D^\mathbf{n}f}{\mathbf{n}!}\right|_{\mathbf{x} =\frac{\mathbf{m}}{N}}\leq 1$, which is a universal constant, we can simplify the expression of $P_{\mathbf{m}}(\mathbf{x})$ as
\begin{equation}
    P_{\mathbf{m}}(\mathbf{x}) = \sum_{\mathbf{n}: |\mathbf{n}|<n}{a_{\mathbf{m},\mathbf{n}}
\left ( \mathbf{x}-\frac{\mathbf{m} }{N}  \right )^\mathbf{n} } .
\end{equation}
Correspondingly, $f_1$ is simplified as 
\begin{equation}
    f_1(\mathbf{x}) = \sum_{\mathbf{m} \in \{0,\dots,N\}^d} \sum_{\mathbf{n}: |\mathbf{n}|<n}{a_{\mathbf{m},\mathbf{n}} \phi_{\mathbf{m}}(\mathbf{x})
\left ( \mathbf{x}-\frac{\mathbf{m} }{N}  \right )^\mathbf{n} },
\label{eq2:expand}
\end{equation}
which is a linear combination of at most $d^n(N+1)^d$ sub-functions. Each sub-function is a product of no more than $d+n-1$ linear univariate factors: $\psi(3Nx_k-3m_k), k=1,\cdots,d$, and at most $n-1$ linear functions $x_k-\frac{m_k}{N}$. 

Now let us construct a quadratic neural network to express $f_1$. As Figure \ref{fig:proofsekcth} shows, we first 
approximate $\phi_{\mathbf{m}}(\mathbf{x})
\left ( \mathbf{x}-\frac{\mathbf{m} }{N}  \right )^\mathbf{n}$, a product of no more than $d+n-1$ terms. Thus, the number of multiplication is $d+n-2$. Suppose $\tilde{Q}$ is the multiplication according to Lemma \ref{prop:xy}, we can express $\phi_{\mathbf{m}}(\mathbf{x}) \left ( \mathbf{x}-\frac{\mathbf{m} }{N}  \right )^\mathbf{n}$ by iteratively applying $\tilde{Q}$:
\begin{equation}
\begin{aligned}
&\tilde{f}_{\mathbf{m}, \mathbf{n}}(\mathbf{x})= \tilde{Q}\bigg{(}\psi\left(3 N x_{1}-3 m_{1}\right),\\
& \tilde{Q}\left(\psi\left(3 N x_{2}-3 m_{2}\right), \ldots, \tilde{Q}\left(x_{k}-\frac{m_{k}}{N}, \ldots\right) \ldots\right)\bigg{)}.\\
\end{aligned}
\end{equation}
Thus, $\tilde{f}_{\mathbf{m}, \mathbf{n}}$ can be constructed by a quadratic ReLU network whose number of parameters is $4(d+n-2)$. Moreover, $\forall k$ and $\mathbf{x} \in [0,1]^d$, using $f_1$ to approximate $\tilde{f}$ is free of error according to Lemma \ref{prop:xy},
\begin{equation}
\scalebox{0.8}{$
\begin{aligned}
&\mid \tilde{f}_{\mathbf{m}, \mathbf{n}}(\mathbf{x})- \phi_{\mathbf{m}}(\mathbf{x})\left(\mathbf{x}-\frac{\mathbf{m}}{N}\right)^{\mathbf{n}} \mid \\
=& \mid \tilde{Q}\left(\psi\left(3 N x_{1}-3 m_{1}\right), \tilde{Q}\left(\psi\left(3 N x_{2}-3 m_{2}\right), \tilde{Q}\left(\psi\left(3 N x_{3}-3 m_{3}\right), \ldots\right)\right)\right) \\
&-\psi\left(3 N x_{1}-3 m_{1}\right) \psi\left(3 N x_{2}-3 m_{2}\right) \psi\left(3 N x_{3}-3 m_{3}\right) \ldots \mid \\
=&\mid\psi\left(3 N x_{1}-3 m_{1}\right) \tilde{Q}\left(\psi\left(3 N x_{2}-3 m_{2}\right), \tilde{Q}\left(\psi\left(3 N x_{3}-3 m_{3}\right), \ldots\right)\right) \\
&-\psi\left(3 N x_{1}-3 m_{1}\right) \psi\left(3 N x_{2}-3 m_{2}\right) \psi\left(3 N x_{3}-3 m_{3}\right) \ldots \mid\\
=&\mid\psi\left(3 N x_{1}-3 m_{1}\right) \psi\left(3 N x_{2}-3 m_{2}\right)\psi\left(3 N x_{3}-3 m_{3}\right), \ldots\\
&-\psi\left(3 N x_{1}-3 m_{1}\right) \psi\left(3 N x_{2}-3 m_{2}\right) \psi\left(3 N x_{3}-3 m_{3}\right) \ldots \mid \\
=& 0.
\end{aligned}$}
\end{equation}

Second, since $\tilde{f}_{\mathbf{m},\mathbf{n}}$ just approximates one term of $f_{1}(\mathbf{x})$, now we aggregate each $\tilde{f}_{\mathbf{m},\mathbf{n}}$ to get a final network:
\begin{equation}
    \tilde{f}=\sum_{\mathbf{m}\in\{0,\dots,N\}^d}\sum_{\n:|\n|<n}a_{\m,\n}\tf_{\m,\n}.
\end{equation}
Because $\tilde{f}_{\mathbf{m},\mathbf{n}}$ comprises of the local term $\phi_{\mathbf{m}}(\mathbf{x})$, $\tilde{f}_{\mathbf{m},\mathbf{n}}$ is also local. Thus, using $\tilde{f}$ to approximate $f_1$ is subjected to zero error:
\begin{equation}
\begin{aligned}
&\left|\tilde{f}(\mathbf{x})-f_{1}(\mathbf{x})\right|\\
\overset{(1)}=&\left|\sum_{\mathbf{m} \in\{0, \ldots, N\}^{d}} \sum_{\mathbf{n}:|\mathbf{n}|<n} a_{\mathbf{m}, \mathbf{n}}\left(\tilde{f}_{\mathbf{m}, \mathbf{n}}(\mathbf{x})-\phi_{\mathbf{m}}(\mathbf{x})\left(\mathbf{x}-\frac{\mathbf{m}}{N}\right)^{\mathbf{n}}\right)\right| \\
\overset{(2)}=&\left|\sum_{\mathbf{m}: \mathbf{x} \in \operatorname{Supp} \phi_{\mathbf{m}}} \sum_{\mathbf{n}:|\mathbf{n}|<n} a_{\mathbf{m}, \mathbf{n}}\left(\tilde{f}_{\mathbf{m}, \mathbf{n}}(\mathbf{x})-\phi_{\mathbf{m}}(\mathbf{x})\left(\mathbf{x}-\frac{\mathbf{m}}{N}\right)^{\mathbf{n}}\right)\right| \\
=&0.
\end{aligned}
\end{equation}
In the above equation, (1) follows from Eq. (\ref{eq2:expand}), and (2) is due to the locality of $\tilde{f}_{\mathbf{m},\mathbf{n}}$. 

From \textbf{Step 1}, we have $\|f-f_1\|_\infty \leq \frac{\epsilon}{2}$. Thus, if we choose $N=\left \lceil \left ( \frac{n!}{2^dd^n} \frac{\epsilon }{2}  \right )^{-1/n} \right \rceil$,
\begin{equation}
\|\tilde{f}-f\|_{\infty} \leq\left\|\tilde{f}-f_{1}\right\|_{\infty}+\left\|f_{1}-f\right\|_{\infty} \leq  0 + \frac{\epsilon}{2} \leq \epsilon.
\end{equation}
By a direct calculation, in total, $\tilde{f}$ consumes no more than $4d^n(N+1)^d(d+n-2)$ parameters.

\end{proof}

\textbf{Remark 1.} Table \ref{tab:size_Euclidean} compares the number of parameters in the conventional and quadratic ReLU networks based on the approximation scheme in \cite{Yarotsky2017Error}. It can be seen that quadratic construction is much more efficient than conventional one. The saving is an order of magnitude. The key to accounting for such a considerable saving is that the quadratic construction can naturally express the multiplication operation. Thus, our analysis can be extended to a general polynomial network to justify their parametric efficiency as well. 

\begin{table}[htbp]
\centering
  \caption{Comparison of the number of parameters used in the conventional and quadratic ReLU networks to approximate functions in Sobolov space based on the approximation scheme in \cite{Yarotsky2017Error}.}
\begin{tabular}{l|l|c}
\hline
Network &    Operation    & \#Parameters  \\ \hline
\multirow{2}{*}{Conventional \cite{Yarotsky2017Error} } & Product           & $\mathcal{O}(\log({1/\epsilon}))$                 \\ \cline{2-3} 
                              & Total & $d^{n}(N+1)^{d} (d+1)n\mathcal{O}(\log (1 / \epsilon))$        \\ \hline
\multirow{2}{*}{Quadratic (ours)} & Product           & 4        \\ \cline{2-3} 
                              & Total & $4d^n(N+1)^d(d+n-2)$         \\ \hline                              
\end{tabular}
\label{tab:size_Euclidean}
\end{table}

\subsection{Efficiency: Functions on a Manifold}
\label{subsec:manifold}

One opinion on why deep learning does not seem to suffer the curse of dimensionality is that many high-dimensional data in the real world lie along a low-dimensional latent manifold in the ambient high-dimensional space \cite{tu2011manifolds}. Per this opinion, it is also necessary to confirm the efficiency of quadratic networks in approximating functions on a manifold instead of Euclidean space. Here, we show that a quadratic network can also much more efficiently approximate functions on low-dimensional manifolds by leveraging that a quadratic neuron can express a product operation very naturally.
Before discussing the main theorem, let us first introduce the concepts relevant to manifolds.

\begin{definition}[Chart]
A chart $(U,\psi)$ of a manifold $\M$ is a pair, where $U \subset \M$ is an open set and $\psi: U \to \R^d$. What's more, both $\psi$ and $\psi^{-1}$ are continuous. 
\end{definition}

The open set $U$ is essentially a local neighborhood of $\M$, and $\psi$ defines a coordinate system on $U$. There might be two charts overlapping with each other. We say two charts $(U_1,\psi_1)$ and $(U_2,\psi_2)$ are $C^k$ compatible if and only if $\psi_1 \circ \psi_2^{-1}: \psi_2(U_1 \cap U_2) \to \psi_1 (U_1 \cap U_2) $ and $\psi_2 \circ \psi_1^{-1}: \psi_1(U_1 \cap U_2) \to \psi_2 (U_1 \cap U_2) $ are both $C^k$ continuous. A chart is insufficient to describe the entire manifold. Therefore, let us define an atlas.

\begin{definition}[$C^k$ atlas]
A $C^k$-atlas is a set $\{(U_i,\psi_i)\}$, wherein any two charts are $C^k$ compatible and $\cup_{i=1}^{C_{\M}}~ U_i = \M$.
\end{definition}

\begin{definition}[Smooth manifold]
A smooth manifold is a manifold with a $C^{\infty}$ atlas.
\end{definition}

\begin{definition}
A manifold is bounded if there exists a constant $B>0$ such that $|\x_i| \leq B$, $i=1,\cdots,d$, for any $\x \in \M$.
\end{definition}

\begin{prop}[$C^{\infty}$ partition of unity]
Let $\{U_i\}_{i=1}^{C_{\M}}$ cover a smooth manifold $\M$, where ${C_{\M}}$ is the number of charts, there exists a partition $\{\rho_i\}_{i=1}^{C_{\M}}$ satisfying
\begin{equation}
    \sum_{i=1}^{C_{\M}} \rho_i =1,
\end{equation}
where the support of $\rho_i$ is within $U_i$.
\end{prop}

\begin{definition}[$C^n$ functions over a manifold]
Let $\M$ be a smooth manifold. If for any chart $(U,\psi)$ of $\M$, the composition $h\circ \psi^{-1}: \psi(U) \to \R$ is continuously differentiable up to the order $n$, then the function $h:\M \to \R$ is $C^n$.
\end{definition}

\begin{definition}[H\"{o}lder space $H^{n, \alpha}$]
A function $f: \mathcal{M} \mapsto \mathbb{R}$ in the Hölder space  $H^{n, \alpha}$ with a positive integer $n$  and  $\alpha \in(0,1]$ satisfies that  $f \in C^{n-1}$  and for any chart  $(U, \phi)$  and $ |\mathbf{n}|=n $, we have
\begin{equation}
\begin{aligned}
& \left|D^{\mathbf{n}}\left(f \circ \phi^{-1}\right)|_{\phi\left(\mathbf{x}_{1}\right)}-D^{\mathbf{n}}f \circ \phi^{-1}|_{\phi\left(\mathbf{x}_{2}\right)} \right| \\
& \leq \left\|\phi\left(\mathbf{x}_{1}\right)-\phi\left(\mathbf{x}_{2}\right)\right\|_{2}^{\alpha}, \quad \forall \mathbf{x}_{1}, \mathbf{x}_{2} \in U
\end{aligned}
\end{equation}
\end{definition}

\begin{definition}[Reach of a manifold, {Definition 2.1 of \cite{aamari2019estimating}}]
Denote the medial axis of a manifold $\mathcal{M}$ as $\mathrm{Med}(\mathcal{M})$ which is the subset of $\mathcal{M}$, and points in $\mathrm{Med}(\mathcal{M})$ have at least two neighbors on $\mathcal{M}$. Denote a distance function as $\mathcal{M}$: $\mathrm{distance}(\z,\mathcal{M})=\inf_{\p\in \mathcal{M}}\|\p-\z\|$.
Then, $\mathrm{Med}(\mathcal{M})=\left\{\z \in \mathbb{R}^D \mid \exists~ \p \neq \q \in \mathcal{M},\|\p-\z\|=\|\q-\z\|=d(\z, \mathcal{M})\right\}$. The reach $\tau > 0$ of $\mathcal{M}$ is defined as
\begin{equation}
    \tau=\inf_{\p \in \mathcal{M}}\mathrm{distance}(\p, \mathrm{Med}(\mathcal{M}))=\inf_{\z \in \mathrm{Med}(\mathcal{M})}\mathrm{distance}(\z,\mathcal{M})
\end{equation}

\end{definition}

Geometrically, one can interpret the reach $\tau$ of a point $\x \in \mathcal{M}$ as the radius of the osculating circle that is at least $\tau$. A large reach means that the manifold does not change rapidly. In our construction, the reach controls the selection of an atlas for $\mathcal{M}$.

\begin{prop}[\textbf{Lemma 2 of \cite{chen2019efficient}}]
For $i=1,\cdots,C_{\M}$, the function $h_i$ belongs to $H^{n,\beta}$ and 
\begin{equation}
\begin{aligned}
        & |D^{\n}(h_i\circ \psi_i^{-1})|_{\psi_i(\x_1)} - D^{\n}(h_i\circ \psi_i^{-1})|_{\psi_i(\x_2)}|  \\
    &\leq L_i ||\psi_i(\x_1)-\psi_i(\x_2)||_2^\beta, ~\forall~ \x_1,\x_2 \in U_i, \\
\end{aligned}
\end{equation}
where $L_i= 2^{n+1} 2\sqrt{d}\mu_i \lambda_i (2r)^{1-\beta}$ is a H\"{o}lder coefficient depending on the dimension $d$ and universal constants $ \mu_i,\lambda_i, r$. 
\label{prop:taylorremainder}
\end{prop}

\begin{prop}[\textbf{Lemma 3 of \cite{chen2019efficient}}]
For any $i\in \{1,2,\cdots,C_{\M}\}$ and $\mathcal{K}_i = \{\x \in \M: r^2-\Delta \leq ||\x-\c_i||_2^2 \leq r^2\}$, there exists a constant $c$ whose value depends on $h_i$ and $\psi_i$ such that
\begin{equation}
    \underset{\x \in \mathcal{K}_i}{\max}~ |h_i(\x)| \leq \frac{c(\pi+1)}{r(1-r/\tau)} \Delta.
\end{equation}
\label{prop:4}
\end{prop}

\begin{theorem}
Given a bounded manifold $\M$ with the positive reach and a function $h: \mathcal{M} \mapsto \mathbb{R} \in H^{n,\beta}$, for any $\epsilon > 0$, there exists a network $\hat{h}$ such that 
\begin{equation}
    \underset{\x \in \M}{\sup} |h(\x)-\hat{h}(\x)|<\epsilon, 
\end{equation}
and $\hat{h}$ has at most $C_{\M}(c_1^{'}\epsilon^{-\frac{d}{n+\beta}})$ parameters, where $C_{\M}$ is the number of partitions on a manifold, $c_1$ is a universal constant. 
\label{manifold_approximation}
\end{theorem}

\textbf{The sketch of proof:} Our strategy is divide-and-conquer. First, we construct an atlas $\{(U_i,\psi_i)\}_{i=1}^{C_{\M}}$ for the manifold of interest. Then, based on such an atlas, a function $h$ in H\"{o}lder space is partitioned and mapped to Euclidean space using $\psi_i^{-1}$. After that, we use a Taylor polynomial $\bar{h}_i$ to approximate it with an arbitrarily small error, and a quadratic network $\hat{f}_i$ can approximate $\bar{h}_i$ with no error. Next, the function $\bar{h}_i$ needs to be back-projected into the original manifold. Hence, we define $\hat{h}_i = \bar{h}_i \circ \psi_i$ and a bounding function $\mathds{1}_{\Delta}$. Then, the final approximator $\hat{h}$ is a sum of products between $\hat{h}_i$ and $\mathds{1}_{\Delta}$. At last, the total error and the number of parameters are counted. Figure \ref{fig:manifold} illustrates our strategy and associated intermediate functions.

\begin{proof} 
\textbf{Step 1: Construct an atlas.} Let the open Euclidean ball in $\R^D$ be $\mathcal{B}(\c,r)$, where $\c$ is the ball center and $r$ is the radius. Because $\mathcal{M}$ is compact, we can find a finite set of points $\{\c_i\}_{i=1}^{C_{\M}}$ and set $r<\tau/2$, $\tau$ is the reach of $\mathcal{M}$, to ensure $\M \subset \cup_i \B(\c_i,r)$. We select $\{(U_i,\psi_i)\}_{i=1}^{C_{\M}}$ as an atlas, where $U_i = \M \cap \B(\c_i,r)$. $C_{\M}$ is upper bounded by
\begin{equation}
C_{\mathcal{M}} \leq\left\lceil\frac{S A(\mathcal{M})}{r^{d}}T_{d}\right\rceil,
\end{equation}
where $SA(\M)$ is the surface area of $\M$ and $T_{d}$ is the maximum thickness of $U_i, i=1,\cdots, C_{\mathcal{M}}$.

At each chart, we define $\psi_i$ as a projection with translating, scaling, and rotation that maps the local neighborhood of a manifold into the Euclidean space $[0,1]^d$. Mathematically, $\psi_i(\x)$ is formulated as
\begin{equation}
    \psi_i(\x) = b_i(\V_i(\x-\c_i)+\textbf{s}_i),
\end{equation}
where $\psi_i(\x) \in [0,1]^d$, $b_i \in (0,1]$ is a scaling factor, $\V_i$ is a rotation matrix, and $\textbf{s}_i$ is a translation vector. $\psi_i(\x)$ can be directly realized by a one-hidden-layer quadratic network.


\textbf{Step 2: Partition and estimation.} We partition $h$ into $\{h_i\}_{i=1}^{C_{\M}}$ such that 
\begin{equation}
    h(\x) = \sum_{i=1}^{C_{\M}} h(\x) \cdot \rho_i (\x) = \sum_{i=1}^{C_{\M}} h_i(\x).
\end{equation}
A quadratic network can approximate $h_i\circ \psi_i^{-1}(\x)$ arbitrarily well. This is because $h_i$ is from the H\"{o}lder class, and we can apply the similar Taylor expansion idea.

Here, let $\bar{h}_i(\x)$ be a Taylor polynomial satisfying the conditions consistent with those in the preceding subsection. We have
\begin{equation}
    \bar{h}_i(\x) = \sum_{\m \in \{1,...,N\}^d}\phi_{\m}{(\x)}P_{\m}{(\x)},
\end{equation}
where, $\phi_{\m}$ is defined in Eq. (\ref{eq2:phimx}) and $P_{\m}{(\x)}$ is an $n$-degree Taylor polynomial of the function $h_i\circ \psi_i^{-1}(\x)$ at $\frac{\m}{N}$ as Eq. (\ref{eq2:taylor}) shows. Thus, the approximation error of $h_i\circ \psi_i^{-1}(\x)$ and $\bar{h}_i(\x)$ is 

\begin{equation}
\begin{aligned}
    & \ \ \ \ \underset{\x \in [0,1] }{\sup} |h_i\circ \psi_i^{-1}(\x)-\bar{h}_i(\x)| \\
    & =  \underset{\x \in [0,1] }{\sup} |\sum_{\m} \phi_{\m} h_i\circ \psi_i^{-1}(\x) - \sum_{\m} \phi_{\m}P_{\m}(\x)|      \\
    & \leq \underset{\x \in [0,1] }{\sup} \sum_{\m} \phi_{\m}|(h_i\circ \psi_i^{-1}(\x) - P_{\m}(\x))| \\
    & \leq \underset{\x \in [0,1] }{\sup} \sum_{{\mathbf{m}:\left|x_{k}-\frac{m_{k}}{N}\right|<\frac{1}{N} \forall k}} |(h_i\circ \psi_i^{-1}(\x) - P_{\m}(\x))| \\
    & \leq \underset{\x \in [0,1] }{\sup} 2^d \underset{{\mathbf{m}:\left|x_{k}-\frac{m_{k}}{N}\right|<\frac{1}{N} \forall k}}{\max} |h_i\circ \psi_i^{-1}(\x) - P_{\m}(\x)| \\
    & \overset{(1)}{\leq}  \frac{2^d d^n}{n!} \Big(\frac{1}{N}\Big)^n \underset{\n=n}{\max} \Big|D^{\n}(h_i\circ \psi_i^{-1})|_{\frac{\m}{N}} - D^{\n}(h_i\circ \psi_i^{-1})|_{\bm{y}} \Big| \\
    & \overset{(2)}{\leq}  \frac{2^d d^n}{n!} \Big(\frac{1}{N}\Big)^n 2^{n+1}\sqrt{d}\mu_i \lambda_i (2r)^{1-\beta} ||\frac{\m}{N}-\x||_{2}^{\beta} \\
    & \leq \frac{2^d d^n}{n!} \Big(\frac{1}{N}\Big)^n 2^{n+1}\sqrt{d}\mu_i \lambda_i (2r)^{1-\beta} d^{2/\beta} \Big(\frac{1}{N}\Big)^\beta,
\end{aligned}
\label{eq2:approximate}
\end{equation}
where (1) follows from the bound of Taylor remainder and $\bm{y}$ is the linear interpolation of $\frac{\m}{N}$ and $\x$; (2) follows from Proposition \ref{prop:taylorremainder}. Similarly, setting 
\begin{equation}
N \geq \Big(\frac{2^d d^n}{\delta n!} 2^{n+1}\sqrt{d}\mu_i \lambda_i (2r)^{1-\beta} d^{2/\beta} \Big)^{\frac{1}{n+\beta}},
\label{eq2:N}
\end{equation}
leads to
\begin{equation}
\frac{2^d d^n}{n!} \Big(\frac{1}{N}\Big)^n 2^{n+1}\sqrt{d}\mu_i \lambda_i (2r)^{1-\beta} d^{2/\beta} \Big(\frac{1}{N}\Big)^\beta \leq \delta.
\label{eq2:delta}
\end{equation}

Then, since $\bar{h}_i(x)$ is constructed by a Tayor polynomial in Euclidean space, a quadratic network $\hat{f}_i$ can approximate it with no error, as we already proved in the preceding subsection. Therefore, the approximation error between $h_i$ and $\hat{f}_i$ is also no more than $\delta$. If we choose $N$ in Eq. (\ref{eq2:N}), a quadratic network needs has at most $c_1\delta^{-\frac{d}{n+\beta}}$ parameters to approximate $\bar{h}_i(x)$, where $c_1= 4d^n(N + 1)^d (d+n-2)$.

\textbf{Step 3: Back-projection and bounding.} The function $\bar{h}_i$ is an approximator to $h_i\circ \psi_i^{-1}$ which is supported on $[0,1]^d$. We need to back-project $\bar{h}_i$ into the manifold in order to approximate our target function $h_i$ which is supported over the $i$-th chart of a manifold. To do so, we compose $\bar{h}_i$ with $\psi_i$ to obtain $\hat{h}_i$: $\hat{h}_i=\bar{h}_i \circ \psi_i$. $\hat{h}_i$ can only approximate $h_i$ well within the domain of $U_i$. Therefore, we need to bound $\hat{h}_i$ within the area of $U_i$. 

Mathematically, we define the truncation function as a composition of a half trapezoid function $\mathds{1}_{\Delta}(s)$ and a radial function $s_i(\x)$. This is because each $U_i$ is included by $\mathcal{B}(\c,r)$. For $i=1,2,\cdots, C_\M$, the radial function is
\begin{equation}
    s_i(\x) = ||\x-c_i||_2^2,
\end{equation}
while the half trapezoid function $\mathds{1}_{\Delta}(s)$ is
\begin{equation}
    \mathds{1}_{\Delta}(s) = 
        \begin{cases}
        1,   \ \ \ \    \ \ \ \       \ \ \ \   \ \ \ \ \ \ \ \ \ \   s < r^2 - \Delta \\
        -\frac{1}{\Delta} (s-r^2), \ \ \ \ \ r^2 - \Delta \leq s \leq r^2\\
        0,    \ \ \ \       \ \ \ \    \ \ \ \   \ \ \ \ \ \ \ \ \ \   s > r^2 ,      \\
        \end{cases}
\end{equation}
where $\Delta$ is chosen based on the total error $\epsilon$. 

\textbf{Step 4: Total error.} To accomplish a truncation, we directly multiply $\hat{h}_i$ with $\mathds{1}_{\Delta}\circ s_i$. The final approximator is $\hat{h}=\sum_{i=1}^{C_\M}\hat{h}_i \times (\mathds{1}_{\Delta}\circ s_i)$. Now we are ready to estimate the total error 

\begin{equation}
    \begin{aligned}
     &   \underset{\x \in \M}{\sup} |h(\x)-\hat{h}(\x)| \\
     = &  \underset{\x \in \M}{\sup} |\sum_{i=1}^{C_{\M}} (h_i(\x) \times \mathds{1}(\x \in U_i) -\hat{h}_i(\x) \times (\mathds{1}_{\Delta}\circ s_i))| \\
     \leq & \underset{\x \in \M}{\sup} \sum_{i=1}^{C_{\M}} | h_i(\x) \times \mathds{1}(\x \in U_i) -\hat{h}_i(\x) \times (\mathds{1}_{\Delta}\circ s_i)| \\
     \leq & \underset{\x \in \M}{\sup} \sum_{i=1}^{C_{\M}} \underbrace{| h_i(\x) \times \mathds{1}(\x \in U_i) - h_i(\x) \times (\mathds{1}_{\Delta}\circ s_i)|}_{A_{i,1}}  \\
         + & \underset{\x \in \M}{\sup} \sum_{i=1}^{C_{\M}} \underbrace{| h_i(\x) \times (\mathds{1}_{\Delta}\circ s_i) -\hat{h}_i(\x) \times (\mathds{1}_{\Delta}\circ s_i)|}_{A_{i,2}} \\
     = & C_{\M}\Big(\frac{c(\pi+1)}{r(1-r/\tau)}\Delta + \delta \Big),     
\end{aligned}
\end{equation}
where $A_{i,1}$ satisfies Proposition \ref{prop:4} and $A_{i,2}$ follows from Eq. (\ref{eq2:approximate})-(\ref{eq2:delta}). Then 
\begin{equation}
|h_i(\x)-\hat{h}_i(\x)|=|h_i\circ \psi^{-1}\circ \psi(\x)-\bar{h}_i\circ \psi(\x)| \leq \delta.
\end{equation}
Finally, we pick $\delta = \frac{\epsilon}{2C_{\M}}$ so that $\sum_{i=1}^{C_{\M}}A_{i,2} \leq \frac{\epsilon}{2}$. Similarly, we choose $\Delta = \frac{r(1-r/\tau)\epsilon}{2c(\pi+1)C_{\M}}$ to make $C_{\M}\frac{c(\pi+1)}{r(1-r/\tau)}\Delta \leq \frac{\epsilon}{2}$.

\begin{figure}
    \centering
    \includegraphics{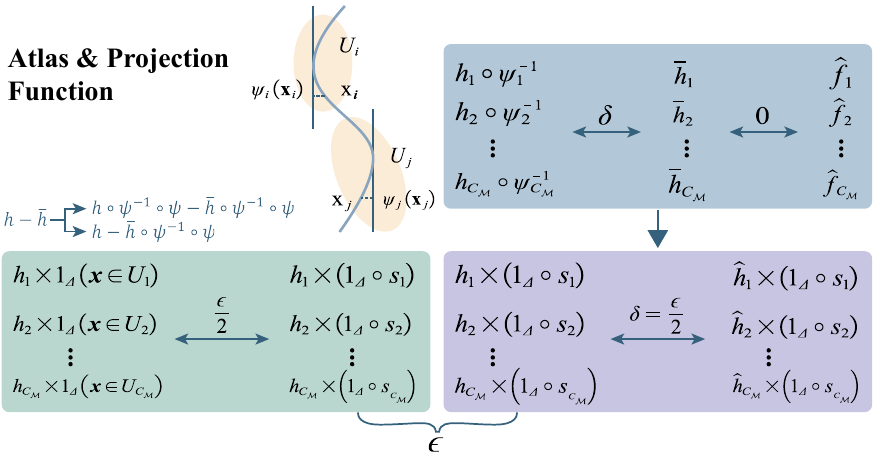}
    \caption{An overview of using a quadratic network to approximate function $h$ on a manifold.}
    \label{fig:manifold}
\end{figure}

Next, let us count the number of parameters in the resultant network. i) For the Taylor polynomial approximator, we calculate in \textbf{Step 2} that there are no more than  $C_{\M}c_1\epsilon^{-\frac{d}{n+\beta}}$ parameters for the quadratic network to approximate $h(\x)$, where $N \geq \Big(\frac{2^{d+1}C_{\M} d^n}{\epsilon n!} 2^{n+1}\sqrt{d}\mu_i \lambda_i (2r)^{1-\beta} d^{2/\beta} \Big)^{\frac{1}{n+\beta}}$. ii) For the back-projection function $\mathds{1}_{\Delta}$, it can be implemented by a single-layer quadratic network. The approximation of the radial function $s_i$ can be implemented by one-hidden-layer quadratic network with $d^{3.75}/\epsilon$ neurons \cite{fan2020universal}. Therefore, the total number of parallel sub-quadratic networks to compute $\mathds{1}_{\Delta}(s)$ has $C_\M$ layers with $C_{\M}d^{3.75}/\epsilon$ parameters. At last, the total number of parameters in the quadratic ReLU network are $C_{\M}(c_1\epsilon^{-\frac{d}{n+\beta}} + d^{3.75}/\epsilon) \approx C_{\M}(c_1^{'}\epsilon^{-\frac{d}{n+\beta}})$. 

\end{proof}

\textbf{Remark 2 (Optimal Error).} Table \ref{tab:manifold} compares the numbers of parameters used in conventional and quadratic ReLU networks in approximating functions in H\"{o}lder space based on the scheme in \cite{chen2019efficient}. It can be seen that the quadratic construction saves the number of parameters exponentially. One may argue that in Subsections \ref{subsec:euclidean} and \ref{subsec:manifold}, we just show that quadratic networks are much more efficient when representing functions with special schemes. However, there are other schemes to represent functions such as using the Fourier transform, and it is not sure if quadratic networks can still achieve parametric efficiency for those schemes. Here, we argue that despite the existence of other kinds of approximation schemes, the representation schemes we are using are the optimal \cite{devore1989optimal} in terms of achieving the lowest error rate. Thus, our theoretical derivations are still well-underpinned because when approximating functions in some space, the approximation scheme with the optimal error rate is usually considered first.

\begin{table}[htbp]
\centering
  \caption{Comparison of the number of parameters used in conventional and quadratic ReLU networks to approximate functions in H\"{o}lder space $H^{n, \beta}$ based on the approximation scheme in \cite{chen2019efficient}. $c_1^{'} < c'$}
 \scalebox{0.88}{ 
\begin{tabular}{l|l|c}
\hline
Network &    Operation    & \#Parameters  \\ \hline
\multirow{2}{*}{Conventional \cite{chen2019efficient} } & Product           & $\mathcal{O}(\ln({1/\epsilon}))$                 \\ \cline{2-3} 
                              & Total & $c'C_{\M}(\epsilon^{-\frac{d}{n+\beta}}\log(1/\epsilon)+d\log(1/\epsilon))$        \\ \hline
\multirow{2}{*}{Quadratic (ours)} & Product           & 18        \\ \cline{2-3} 
                              & Total & $c_1^{'}C_{\M}\epsilon^{-\frac{d}{n+\beta}}$        \\ \hline                              
\end{tabular}}
\label{tab:manifold}
\end{table}

\subsection{Efficiency: Barron Space Analysis}

From the perspective of Barron space, we justify the high parametric efficiency of quadratic networks by showing that they can avoid the curse of dimensionality. In neural network approximation theory, to achieve the same level of error, the complexity of the constructed network structure usually increases with the dimension of functions \cite{shen2019deep, lu2021deep, Yarotsky2017Error}. This phenomenon is referred to as the curse of dimensionality. However, functions in the Barron space can be expressed by a one-hidden-layer network with an error rate that is not dimension-dependent but controlled by a certain norm of the target function. The Barron space of conventional neurons, denoted as $\mathcal{B}_2^{(1)}$, was derived in \cite{ma2022barron, barron1993universal}. Here, we show that quadratic neurons also possess this dimension-independence property by deriving the Barron space of quadratic neurons, denoted as $\mathcal{B}_2^{(2)}$. Then, since $\mathcal{B}_2^{(2)}$ comprises of $\mathcal{B}_2^{(1)}$, using quadratic neurons to approximate functions in $\mathcal{B}_2^{(2)}\setminus\mathcal{B}_2^{(1)}$ is dimension-free, but using conventional neurons is dimension-dependent, thereby demonstrating the parametric efficiency of quadratic neurons in terms of avoiding the curse of dimensionality.

\begin{definition}[Barron space of conventional neurons, $\mathcal{B}_{2}^{(1)}$ \cite{ma2022barron}]
This space is made of functions $f: X \mapsto \mathbb{R}$ that admit the following representation:
\begin{equation}
    f(\mathbf{x})=\int_{\Omega} a \sigma\left(\boldsymbol{b}^{T} \mathbf{x}+c\right) \rho(d a, d \boldsymbol{b}, d c), \quad \mathbf{x} \in X,
    \label{barron_space_l}
\end{equation}
where  $\Omega=\mathbb{R}^{1} \times \mathbb{R}^{d} \times \mathbb{R}^{1}$, $\rho$ is a probability distribution on  $\left(\Omega, \Sigma_{\Omega}\right)$ with $\Sigma_{\Omega}$  being a Borel  $\sigma$-algebra on $\Omega$, and $\sigma(x)=\max \{x, 0\}$  is the ReLU activation function. We define the Barron norm of the representation Eq. \eqref{barron_space_l} as
\begin{equation}
    \|f\|_{\mathcal{B}_{2}^{(1)}}=\inf _{\rho}\left(\mathbb{E}_{\rho}\left[|a|^{2}\left(\|\boldsymbol{b}\|_{1}+|c|\right)^{2}\right]\right)^{1 / 2}.
\end{equation}
The infimum is taken over all  $\rho$  for which Eq. \eqref{barron_space_l} holds for all $\mathbf{x} \in X$. The Barron space $\mathcal{B}_{2}^{(1)}$ constitutes the set of continuous functions defined by Eq. \eqref{barron_space_l} with a finite Barron norm.

\end{definition}

Note that Eq. \eqref{barron_space_l} can be analogized to two-layer neural networks of $m$ neurons when $m\to \infty$:
\begin{equation}
    f_{m}(\mathbf{x} ; \Theta):=\frac{1}{m} \sum_{j=1}^{m} a_{j} \sigma\left(\boldsymbol{b}_{j}^{T} \mathbf{x}+c_{j}\right),
\end{equation}
where $\Theta=\left(a_{1}, \boldsymbol{b}_{1}, c_{1}, \ldots, a_{m}, \boldsymbol{b}_{m}, c_{m}\right) $ is a collection of all parameters. This is an intuition of why using conventional neurons to approximate functions in Eq. \eqref{barron_space_l} is highly efficient \cite{ma2022barron}. In addition, the Barron space is quite inclusive. According to Proposition 1 in \cite{barron1993universal}, given a function, as long as its Fourier distribution has an integrable
density as well as a finite first moment, it belongs to the Barron space. Therefore, the Barron space already fulfills the needs of most real-world tasks. Due to the inclusiveness of the Barron space, our parametric efficiency result based on the Barron space makes general sense.

\begin{prop}[Theorem 1 in \cite{ma2022barron}]
For any $f \in \mathcal{B}_{2}^{(1)}$ and $m>0$, there exists a two-layer neural network $f_{m}(\mathbf{x} ; \Theta)=   \frac{1}{m} \sum_{k=1}^{m} a_{k} \sigma\left(\boldsymbol{b}_{k}^{T} \mathbf{x}+c_{k}\right)$ such that
\begin{equation}
    \left\|f(\cdot)-f_{m}(\cdot ; \Theta)\right\|^{2} \leq \frac{3\|f\|_{\mathcal{B}_{2}^{(1)}}^{2}}{m}.
\end{equation}
\label{forward}
\end{prop}

The inverse approximation theorem can also be proved. To claim the inverse approximation theorem, the following is defined:

\begin{equation}
\begin{aligned}
     & \mathcal{N}_{Q}^{(1)}=\{\frac{1}{m} \sum_{k=1}^{m} a_{k} \sigma\left(\boldsymbol{b}_{k}^{T} \mathbf{x}+c_{k}\right): \\
    & \frac{1}{m} \sum_{k=1}^{m}\left|a_{k}\right|\left(\left\|\boldsymbol{b}_{k}\right\|_{1}+\left|c_{k}\right|\right) \leq Q, m \in \mathbb{N}^{+} \}.  
\end{aligned}
\end{equation}

\begin{prop}[Theorem 2 in \cite{ma2022barron}]
Let $f^{*}$ be a continuous function on  $X$. Assume there exists a constant $Q$ and a sequence of functions  $\left(f_{m}\right) \subset \mathcal{N}_{Q}^{(1)}$ such that $f_{m}(\mathbf{x}) \rightarrow f^{*}(\mathbf{x})$
for all  $\mathbf{x} \in X$. Then there exists a probability distribution  $\rho^{*}$ on $\left(\Omega, \Sigma_{\Omega}\right)$ such that
\begin{equation}
    f^{*}(\mathbf{x})=\int a \sigma\left(\boldsymbol{b}^{T} \mathbf{x}+c\right) \rho^{*}(d a, d \boldsymbol{b}, d c),
\end{equation}
for all $ \mathbf{x} \in X$. Furthermore, we have  $f^{*} \in \mathcal{B}_2^{(1)}$  with
\begin{equation}
    \left\|f^{*}\right\|_{\mathcal{B}_2^{(1)}} \leq Q.
\end{equation}
\label{inverse}
\end{prop}

Propositions \ref{forward} and \ref{inverse} show that the Barron space $\mathcal{B}_2^{(1)}$ is the space specifically for a two-layer network with conventional neurons that admits the direct and inverse approximation theorems. Approximating functions in $\mathcal{B}_2^{(1)}$ using a two-layer conventional network suffers no curse of dimensionality. In the same vein, we can also derive the Barron space for a two-layer quadratic network. Now let us first define the Barron space of quadratic neurons:

\begin{definition}[Barron space of quadratic neurons, $\mathcal{B}_{2}^{(2)}$]
This space is made of functions $g: X \mapsto \mathbb{R}$ that admit the following representation:
\begin{equation}
\scalebox{0.92}{$
\begin{aligned}
        &g(\mathbf{x})=  \\
    &\int_{\Omega} a \sigma((\w_1^{\top}\x+b_1)(\w_2^{\top}\x+b_2)+\w_3^{\top}(\x\odot\x)+c) \rho(d a, d \w, d \b),
    \label{barron_space}
\end{aligned}$}
\end{equation}
where $\mathbf{x} \in X$, $\Omega=\mathbb{R}^{1} \times \mathbb{R}^{3d} \times \mathbb{R}^{3}$, $\w=(\w_1,\w_2,\w_3)$, $\b=(b_1,b_2,b_3)$, $\rho$ is a probability distribution on $\left(\Omega, \Sigma_{\Omega}\right)$, with $\Sigma_{\Omega}$  being a Borel  $\sigma$-algebra on $\Omega$, and $\sigma(x)=\max \{x, 0\}$  is the ReLU activation function. We define the Barron norm of the representation Eq. \eqref{barron_space} as the following:
\begin{equation}
\scalebox{0.95}{$
\begin{aligned}
    & \|g\|_{\mathcal{B}_{2}^{(2)}}= \\
    & \inf _{\rho}\left(\mathbb{E}_{\rho}\left[|a|^{2}(\left(\|\w\|_{1}+|b_1|\right)\left(\|\w_2\|_{1}+|b_2|\right)+\|\w_3\|_{1}+|b_3|)^{2}\right]\right)^{1 / 2}, 
\end{aligned}$}   
\end{equation}
The infimum is taken over all  $\rho$  for which Eq. \eqref{barron_space} holds for all $\mathbf{x} \in X$. The Barron space $\mathcal{B}_{2}^{(2)}$ is the set of continuous functions defined by Eq. \eqref{barron_space} with a finite Barron norm.

\end{definition}

Now, we prove the direct approximation theorem for a two-layer quadratic neural network. 

\begin{theorem}[Direct Theorem of Quadratic Neurons]
For any $g \in \mathcal{B}_2^{(2)}$  and  $m>0$ , there exists a two-layer neural network
    \begin{equation}
    \begin{aligned}
              & g_{m}(\mathbf{x} ; \Theta) \\
          =  & \frac{1}{m} \sum_{k=1}^{m} a_k \sigma((\w_{1k}^{\top}\x+b_{1k})(\w_{2k}^{\top}\x+b_{2k})+\w_{3k}^{\top}(\x\odot\x)+b_{3k}), 
    \end{aligned}
    \end{equation}
where $\Theta$ denotes a collection of parameters $\left\{\left(a_{k}, \w_{k}, \b_{k}\right), k \in[m]\right\}$ in the neural network such that
\begin{equation}
    \left\|g(\cdot)-g_{m}(\cdot ; \Theta)\right\|^{2} \leq \frac{2\|g\|_{\mathcal{B}_2^{(2)}}^{2}}{m}.
\end{equation}
\label{forward_q}
\end{theorem}

\begin{proof}
Let $\rho$ be a probability distribution such that 
\begin{equation}
    g(\mathbf{x})=\mathbb{E}_{\rho}[a \sigma((\w_1^{\top}\x+b_1)(\w_2^{\top}\x+b_2)+\w_3^{\top}(\x\odot\x)+b_3)]
\end{equation}
and 
\begin{equation}
\begin{aligned}
& \mathbb{E}_{\rho}\left[|a|^{2}(\left(\|\w\|_{1}+|b_1|\right)\left(\|\w_2\|_{1}+|b_2|\right)+\|\w_3\|_{1}+|b_3|)^{2}\right] \\
\leq &(1+\varepsilon)\|g\|_{\mathcal{B}_{2}}^{2}, 
\end{aligned}
\end{equation}
where $\varepsilon$ will be determined later. Let $\phi(\mathbf{x}; \Theta)=a \sigma((\w_1^{\top}\x+b_1)(\w_2^{\top}\x+b_2)+\w_3^{\top}(\x\odot\x)+b_3)$  with  $\Theta=(a, \w, \b) \sim \rho$. Then we have $\mathbb{E}_{\Theta \sim \rho}[\phi(\mathbf{x} ; \Theta)]=g(\mathbf{x})$. Let  $\Theta=\left\{\Theta_{j}\right\}_{j=1}^{m}$ be \textit{i.i.d.} random variables drawn from  $\rho(\cdot)$, and consider the empirical average of $\phi(\mathbf{x}; \Theta)$:
\begin{equation}
    \hat{g}_{m}(\mathbf{x};\Theta)=\frac{1}{m} \sum_{j=1}^{m} \phi\left(\mathbf{x};\theta_{j}\right).
\end{equation}

Now let us calculate the approximation error $\mathcal{E}(\Theta)=\mathbb{E}_{\mathbf{x}}\left[\left|\hat{g}_{m}(\mathbf{x};\Theta)-g(\mathbf{x})\right|^{2}\right]$. Its expectation with respect to $\Theta$ is
\begin{equation}
\begin{aligned}
&~~~~~\mathbb{E}_{\Theta}[\mathcal{E}(\Theta)]=\mathbb{E}_{\Theta} \mathbb{E}_{\mathbf{x}}\left|\hat{g}_{m}(\mathbf{x} ; \Theta)-f(\mathbf{x})\right|^{2} \\
& =\mathbb{E}_{\mathbf{x}} \mathbb{E}_{\Theta}\left|\frac{1}{m} \sum_{j=1}^{m} \phi\left(\mathbf{x} ; \Theta_{j}\right)-g(\mathbf{x})\right|^{2} \\
& =\frac{1}{m^{2}} \mathbb{E}_{\mathbf{x}} \sum_{j, k=1}^{m} \mathbb{E}_{\Theta_{j}, \Theta_{k}}\left[\left(\phi\left(\mathbf{x} ; \Theta_{j}\right)-f(\mathbf{x})\right)\left(\phi\left(\mathbf{x} ; \Theta_{k}\right)-g(\mathbf{x})\right)\right] \\
& \leq \frac{1}{m^{2}} \sum_{j=1}^{m} \mathbb{E}_{\mathbf{x}} \mathbb{E}_{\Theta_{j}}\left[\left(\phi\left(\mathbf{x} ; \Theta_{j}\right)-g(\mathbf{x})\right)^{2}\right] \\
& \leq \frac{1}{m} \mathbb{E}_{\mathbf{x}} \mathbb{E}_{\Theta \sim \rho}\left[\phi^{2}(\mathbf{x};\Theta)\right] \\
& \leq \frac{(1+\varepsilon)\|g\|_{\mathcal{B}_{2}}^{2}}{m}.
\end{aligned}
\end{equation}

Define the event $E=\left\{\mathcal{E}(\Theta)<\frac{2\|g\|_{\mathcal{B}_{2}}^{2}}{m}\right\}$. By Markov inequality, we have
\begin{equation}
    \mathbb{P}\left\{E\right\}=1-\mathbb{P}\left\{E_{c}\right\} \geq 1-\frac{\mathbb{E}_{\Theta}[\mathcal{E}(\Theta)]}{2\|g\|_{\mathcal{B}_{2}}^{2} / m} \geq \frac{1-\varepsilon}{2}>0.
\end{equation}
Let $\epsilon<1$, then $\mathbb{P}\left\{E\right\}>0$. Therefore, by choosing any $\Theta$ in $E$, the network defined by this $\Theta$ concludes the proof.
\end{proof}

Now, we prove the inverse approximation theorem for a two-layer quadratic network. First, for convenience, we prescribe that when $m$ goes to infinity, the output function of a one-hidden-layer quadratic network is still bounded.

\begin{equation}
\scalebox{0.92}{$
\begin{aligned}
     &\mathcal{N}_{Q}^{(2)} \\
     =&\{\frac{1}{m} \sum_{k=1}^{m} a_k \sigma((\w_{1k}^{\top}\x+b_{1k})(\w_{2k}^{\top}\x+b_{2k})+\w_{3k}^{\top}(\x\odot\x)+b_{3k}): \\
    & \frac{1}{m} \sum_{k=1}^{m} \left|a_k\right|\left((\Vert\w_{1k}\Vert_1+\left|b_{1k}\right|)\left(\left\|\w_{2k}\right\|_{1}+\left|b_{2k}\right|\right)+\Vert\w_{3k}\Vert_1+|b_{3k}|\right) \\
    & \leq Q, m \in \mathbb{N}^{+} \}.  
\end{aligned}$}
\end{equation}

\begin{theorem}
Let $g^{*}$ be a continuous function on $X$. Assume there exists a constant $Q$ and a sequence of functions $\left(g_{m}\right) \subset \mathcal{N}_{Q}^{(2)}$ such that
$g_{m}(\mathbf{x}) \rightarrow g^{*}(\mathbf{x})$
for all $\mathbf{x} \in X$. Then there exists a probability distribution  $\rho^{*}$ on $\left(\Omega, \Sigma_{\Omega}\right)$ such that
\begin{equation}
\scalebox{0.92}{$
\begin{aligned}
    & g^{*}(\mathbf{x})\\
    =&\int_{\Omega} a \sigma((\w_1^{\top}\x+b_1)(\w_2^{\top}\x+b_2)+\w_3^{\top}(\x\odot\x)+b_3) \rho(d a, d \w, d \b),
\end{aligned}$}
\end{equation}
for all $\mathbf{x} \in X$. Furthermore, we have  $g^{*} \in \mathcal{B}_2^{(2)}$  with
\begin{equation}
    \left\|g^{*}\right\|_{\mathcal{B}_2^{(2)}} \leq Q. 
\end{equation}
\label{backward_q}
\end{theorem} 

\begin{proof}

Let $\Theta_{m}=\left\{\left(a_{k}^{(m)}, \w_{1k}^{(m)}, \w_{2k}^{(m)}, \w_{3k}^{(m)}, \b_{k}^{(m)}\right)\right\}_{k=1}^{m}$ be parameters in $g_{m}$, and let  $A=\sum_{k=1}^{m}\left|a_{k}^{(m)}\right|$ and $\alpha_{k}=\frac{\left|a_{k}^{(m)}\right|}{A}$. Then we can define a probability measure:
\begin{equation}
\scalebox{0.92}{$
\begin{aligned}
        & \rho_{m} \\
        =&\sum_{k=1}^{m} \alpha_{k} \delta\left(a-\frac{\operatorname{sign}\left(a_{k}^{(m)}\right) A}{m}\right) \delta\left(\w-\w_{k}^{(m)}\right) \delta\left(\boldsymbol{b}-\boldsymbol{b}_{k}^{(m)}\right),
\end{aligned}$}
\end{equation}
which satisfies

\begin{equation}
g_{m}\left(\mathbf{x} ; \Theta_{m}\right)=\int a \sigma\left(\boldsymbol{b}^{T} \mathbf{x}+c\right) \rho_{m}(d a, d \w, d \b).
\end{equation}
We can construct $(a, \w, \b)$ from the following set 
\begin{equation}
\begin{aligned}
& \{(a, \w, \b):|a_k| \leq Q, \\
&(\Vert\w_{1k}\Vert_1+\left|b_{1k}\right|)\left(\left\|\w_{2k}\right\|_{1}+\left|b_{2k}\right|\right)+\Vert\w_{3k}\Vert_1+|b_{3k}| \leq 1\} \\
\end{aligned}
\end{equation}
such that for all $m$, $\operatorname{supp}\left(\rho_{m}\right) \subset \mathcal{N}_{Q}^{(2)}$. Since $\mathcal{N}_{Q}^{(2)}$ is compact, the sequence $\left(\rho_{m}\right)$ is tight, as $m \to \infty$. By Prokhorov's Theorem \cite{prokhorov1956convergence}, there exists a subsequence  $\left(\rho_{m_{t}}\right)$ that weakly converges to a probability measure  $\rho^{*}$.
What's more, since $\operatorname{supp}\left(\rho_{m}\right) \subset \mathcal{N}_{Q}^{(2)}$, we also have  $\operatorname{supp}\left(\rho^{*}\right) \subset \mathcal{N}_{Q}^{(2)}$. Therefore, we have
\begin{equation}
\left\|g^{*}\right\|_{\mathcal{B}} \leq Q .
\end{equation}
Furthermore, since $\rho^{*} $ is the weak limit of $\rho_{m_{k}}$, we have

\begin{equation}
\scalebox{0.9}{$
\begin{aligned}
    & g^{*}(\mathbf{x}) \\
    =&\lim _{k \rightarrow \infty} \int_{\Omega} a \sigma((\w_1^{\top}\x+b_1)(\w_2^{\top}\x+b_2)+\w_3^{\top}(\x\odot\x)+c) d \rho_{m_{t}} \\
    = & \int a \sigma((\w_1^{\top}\x+b_1)(\w_2^{\top}\x+b_2)+\w_3^{\top}(\x\odot\x)+c)  d \rho^{*}(d a, d \w, d\b).
\end{aligned}$}
\end{equation}

\end{proof}

\begin{theorem}
For any $g \in \mathcal{B}_2^{(2)}\setminus\mathcal{B}_1^{(2)}$, approximating $g$ with a two-layer quadratic network can avoid the curse of dimensionality, while approximating $g$ with a two-layer conventional network cannot.
\label{efficiency_barron}
\end{theorem}

\begin{proof}
    Combining Propositions \ref{forward} and \ref{inverse}, $\mathcal{B}_2^{(2)}$ is the minimum space that a two-layer conventional network can approximate without the curse of dimensionality. Combining Theorems \ref{forward_q} and \ref{backward_q}, $\mathcal{B}_2^{(2)}$ is the minimum space that a two-layer quadratic network can approximate without the curse of dimensionality, which leads to the proof.
\end{proof}

\textbf{Remark 3.} Theoretical results from perspectives of Sobolov space,  manifolds, and Barron space are encouraging. We think that the parametric efficiency of quadratic networks should hold in general because our established theoretical derivation has no unrealistic assumptions, and the investigated functional spaces are also inclusive. However, we argue that how much efficiency can be gained is genuinely dependent on the task. On the one hand, our theoretical results are constructive. The quadratic network may not learn the mapping as our construction suggests because the optimization of a neural network is a highly non-convex issue. On the other hand, our theoretical derivation is related to the approximation ability characterization. In real-world tasks, the generalization ability of a network should also be taken into account. For example, if well-curated big data are not available, the overfitting issue may arise such that the efficiency of quadratic networks may not be achieved.

\section{Synthetic Experiments}

To put our work in perspective, we construct two insightful synthetic examples (high-dimensional concentric hyperspheres and Gaussian mixture data) to illustrate the efficiency of quadratic networks. \ff{The parametric efficiency gained in both experiments is because nonlinear quadratic terms in quadratic neurons can conveniently represent nonlinear decision boundaries.} The high-dimensional concentric hypersphere experiments highlight that using quadratic networks is free of the curse of dimensionality, which corroborates our Barron space analysis. The Gaussian mixture data experiments implicate that the efficiency of quadratic networks holds for broad tasks since the Gaussian mixture model (GMM) is a so widely-used model to describe data distribution, and data classification is one of the foundational tasks in machine learning. Note that the previous work in \cite{qi2022superiority} has already used quadratic networks to classify the Gaussian mixture data in 2-dimensional space. Here, we generalize this work into high-dimensional space, and our emphasis is efficiency.

\subsection{Concentric Hyperspheres}

We generate two concentric hyperspheres in $d$-dimensional space. Each has $2,000$ samples slightly perturbed by Gaussian noise with a mean of $0$ and a standard deviation of $0.03$. The radii of two hyperspheres are 1 and 0.7, respectively. We assume that two hyperspheres come from two classes. A total of $4,000$ samples are shuffled and split into training and test sets according to the ratio of $0.8:0.2$. We employ one-hidden-layer ReLU conventional and quadratic networks to classify concentric hyperspheres, respectively. For both networks, the learning rate is set to 0.01; the batch size is set to 64; the training epoch number is set to 50. We adopt Adam as an optimizer. Experiments are performed on an RTX 3080 Ti 12GB GPU and an Intel i9 10900k CPU.

\begin{table}[htbp]
\centering
\caption{Classification results on concentric hyperspheres and the number of quadratic and conventional neurons used in classification.}
\scalebox{0.95}{
\begin{tabular}{l|ccccc}
\hline
Dimensions     & 3               & 10              & 20              & 100             & 200             \\ \hline 
\# Quadratic neurons    & \textbf{1}      & \textbf{1}      & \textbf{1}      & \textbf{1}      & \textbf{1}      \\ 
Accuracy (\%) & \textbf{100.00} & \textbf{100.00} & \textbf{100.00} & \textbf{100.00} & \textbf{100.00} \\ 
\# Conventional neurons    & 8               & 40              & 150             & 350             & 700             \\ 
Accuracy (\%) & \textbf{100.00} & \textbf{100.00} & 99.75           & 88.25           & 82.75           \\ \hline
\end{tabular}}
\label{tab:circle}
\end{table}

First, classification results of concentric hyperspheres are shown in Table \ref{tab:circle}, wherein the efficiency of quadratic networks over the conventional counterpart is fully manifested. Regardless of dimensions, one quadratic neuron is always sufficient to classify concentric hyperspheres perfectly \ff{because the power term in a quadratic neuron can naturally form a hypersphere boundary}. In contrast, conventional networks cannot realize a complete classification in high-dimensional space, though it uses hundreds of neurons. This is due to the curse of dimensionality, \textit{i.e.}, in high-dimensional space, it takes an exponential number of conventional neurons to generate a closed boundary to encompass the inner hypersphere. Next, as visualized in Figure \ref{fig:syn}, the decision boundary of a single quadratic neuron  is a smooth circle, but the single-neuron conventional one fails. While a six-neuron conventional network can do the job, its decision boundary is a piecewise linear function, unnatural for the rings.


\begin{figure}[htbp]
    \centering
    \includegraphics[width=\linewidth]{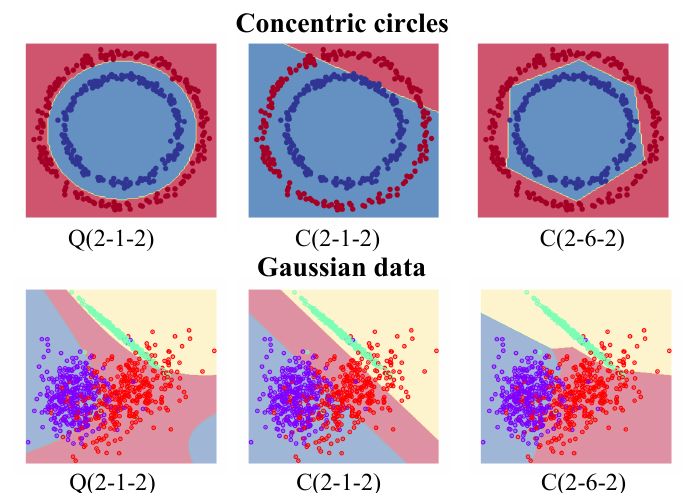}
    \caption{Decision boundaries of quadratic and conventional neural network for classifying circles and Gaussian mixture data. Q(2-$x$-2) is a one-hidden-layer quadratic network with $x$ neurons, while C(2-$x$-2) is a one-hidden-layer conventional network with $x$ neurons.}
    \label{fig:syn}
\end{figure}

\subsection{Gaussian Mixture Data}

We generate multi-category Gaussian mixture data by utilizing the function $make\_classification()$ in the scikit-learn library \cite{pedregosa2011scikit} of Python. The Gaussian distribution of each cluster is with a mean of $0$ and a standard deviation of $1$. A total of 5000 samples are divided into 3 classes in two-dimensional space for the purpose of visualization and 10 classes in higher-dimensional space to simulate complicated situations. We employ the fully-connected one-hidden-layer ReLU neural networks to classify the Gaussian mixture data. The learning rate is set to 0.01. The batch size is set to 64. The number of training epochs is 200. For quadratic networks, we adopt the ReLinear training strategy \cite{fan2021expressivity}, where the learning rate of quadratic terms is set to $1\times10^{-4}$. We adopt Adam as an optimizer. Experiments are performed on an RTX 3080 Ti 12GB GPU and an Intel i9 10900k CPU.

\begin{table}[h]
\centering
\caption{Classification results on high-dimensional Gaussian mixture data. }
\begin{tabular}{lccc}
\hline
\multicolumn{1}{l}{Structure}        & \multicolumn{1}{c}{\#Param} & \multicolumn{1}{c}{\#FLOPs} & ACC(\%)         \\ \hline 
\multicolumn{1}{l}{C(20-150-10)}     & \multicolumn{1}{c}{4.66K}  & \multicolumn{1}{c}{4.50K}  & 79.60          \\ 
\multicolumn{1}{l}{C(20-150-100-10)} & \multicolumn{1}{c}{19.26K} & \multicolumn{1}{c}{19.00K} & 79.20          \\ 
\multicolumn{1}{l}{Q(20-30-10)}      & \multicolumn{1}{c}{2.82K}  & \multicolumn{1}{c}{2.70K}  & \textbf{82.00} \\ 
\multicolumn{1}{l}{C(500-90-10)}     & \multicolumn{1}{c}{46.00K} & \multicolumn{1}{c}{45.90K} & 63.80          \\ 
\multicolumn{1}{l}{C(500-120-10)}    & \multicolumn{1}{c}{61.33K} & \multicolumn{1}{c}{61.20K} & 66.60          \\ 
\multicolumn{1}{l}{Q(500-30-10)}     & \multicolumn{1}{c}{46.02K} & \multicolumn{1}{c}{45.90K} & \textbf{67.40} \\ \hline
\end{tabular}
\label{tab:gaussian}
\end{table}

First, as shown in Figure \ref{fig:syn}, using quadratic neurons can naturally form smooth nonlinear decision boundaries, which achieves more precise distinguishment than conventional neurons whose boundaries are piecewise linear. The conventional network requires more neurons to realize the same level of precise distinguishment. Furthermore, classification results are shown in Table~\ref{tab:gaussian}. Given the same number of neurons, quadratic networks always outperform conventional ones, \textit{i.e.}, the accuracy of quadratic networks to classify 20-dimensional Gaussian data is 7.2\% higher than its counterpart. On the 500-dimensional data, the quadratic one is 10\% higher. Even when much more neurons are used, the conventional networks' performance is still inferior to quadratic networks. Considering that the Gaussian mixture hypothesis is widely deemed effective in machine learning, our results suggest that a quadratic network can save parameters in broad applications. \ff{Moreover, such a saving is because the quadratic operation in a quadratic network can directly represent a smooth nonlinear decision boundary.}

\section{Public Benchmark Experiments}

\subsection{Point Cloud Segmentation}

Recently, point clouds have gained lots of traction due to their broad applications such as autonomous driving \cite{cui2021deep} as a fundamental object representation in three-dimensional space. Here, we evaluate if the parametric efficiency of quadratic networks holds over point cloud-related tasks. We select S3DIS \cite{armeni20163d} to test, which is a widely-adopted point cloud semantic segmentation benchmark. The proposed quadratic network is a drop-in replacement to the classic models: DGCNN \cite{dgcnn} and PointNet \cite{qi2017pointnet}. For fair comparisons, we only replace the linear aggregation with the quadratic one without changing other machinery. \ff{All hyperparameters for training are the same as PointNet and DGCNN.} For S3DIS, following the previous work \cite{dgcnn, qi2017pointnet, zhang-shellnet-iccv19, qi2017pointnet++}, the batch size is set to 32. We train our model for 100 epochs with 4 Tesla V100 GPUs. Following standard practice, the raw input points are firstly grid-sampled to generate $4,096$ points. Unless otherwise specified, we use scale and jitter as data augmentation. For ScanNet, we train for 100 epochs with a weight decay of 0.1 and a batch size of 32, respectively. The number of input points is set to be $8,192$ by sampling. Except for random jitter, data augmentation is the same with S3DIS. We use the average accuracy (mAcc) and IOU (mIOU) as evaluation metrics.

\begin{table}[!htb]
    \centering
        \caption{Point cloud segmentation results over S3DIS and the corresponding model sizes. }
    \begin{tabular}{lccc}
    \hline
    Model & $\#$ Params & mAcc & mIOU \\
    \hline 
    PointNet (Original) & 3.53M & 53.1 & 45.0 \\
    PointNet (Wider) & 3.95M & 53.4 & 45.2\\
    PointNet (Deeper) & 3.87M & 54.9 & 46.1\\
    PointNet+Ours & 3.81M & \textbf{56.5} & \textbf{48.1}\\
    \hline 
    DGCNN (Original) & 0.98M & 57.0 & 49.3\\
    DGCNN (Wider) & 1.53M & 57.5 & 49.6\\
    DGCNN (Deeper) & 1.45M & 59.1 & 49.9\\
    DGCNN+Ours & 1.42M & \textbf{61.9} & \textbf{50.6}\\
    \hline
    \end{tabular}
    \label{tab: complexity}
\end{table}

The point cloud segmentation results over S3DIS and the corresponding model sizes are shown in Table \ref{tab: complexity}, where we increase the number of parameters in  PointNet \cite{qi2017pointnet} and DGCNN \cite{dgcnn} by either using more feature channels (`Wider') or stacking more layers (`Deeper'). The results show that the quadratic network outperforms the wider and deeper PointNet and DGCNN, although wider and deeper PointNet and DGCNN have more parameters. More favorably, the performance of quadratic networks leads its counterparts by a large margin, \textit{e.g.}, 2.8\% mAcc gain over the deeper DGCNN and 3.1\% mAcc gain over the wider PointNet. Figure \ref{fig: example} visualizes the point cloud segmentation result by different models. It is seen that models empowered with quadratic neurons show clearer segmentation. \ff{Because the number of parameters in quadratic networks is smaller, we think that the performance improvement is given rise by nonlinear aggregation in the quadratic network that facilitates the extraction of useful nonlinear features.}

\begin{figure}[!tb]
    \centering
    \includegraphics[width=0.85\linewidth]{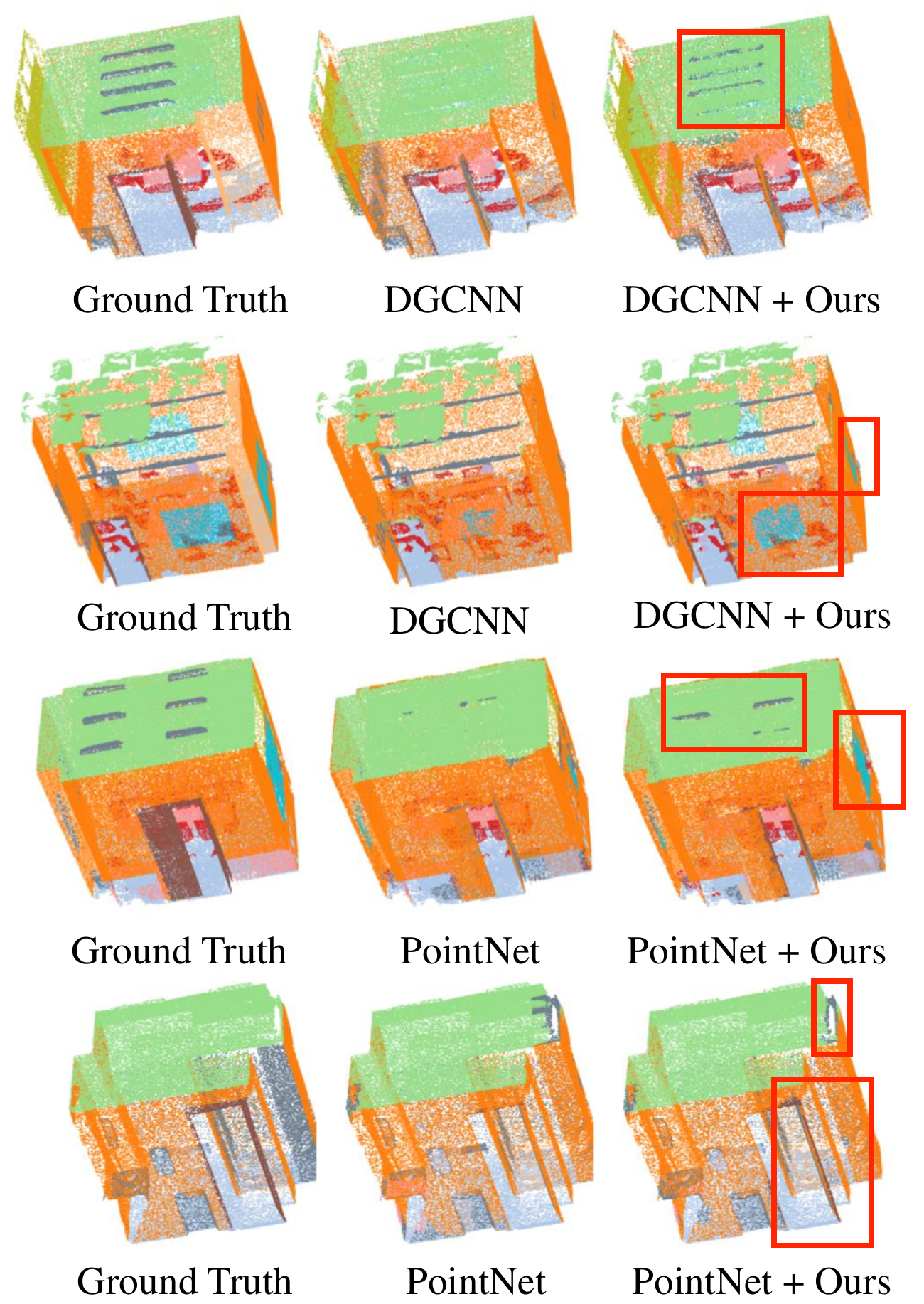}
    \caption{The model empowered with quadratic neurons show superior segmentation accuracy over the original one.}
    \label{fig: example}
\end{figure}


\subsection{Image Recognition}
Image recognition over the ImageNet \cite{deng2009imagenet} is the gold standard to evaluate a model. Here we also apply quadratic networks to the image recognition task to verify the parametric efficiency. 
The ImageNet-1K \cite{deng2009imagenet} benchmark has 1.28M training images and 50K validation images from 1K classes. We use the MMCLS \cite{2020mmclassification} as the basic implementation environment. We employ an AdamW optimizer for 100 epochs using a cosine decay learning rate scheduler and 20 epochs of linear warm-up. A batch size of 1024, an initial learning rate of 0.001, and a  weight decay of 0.05 are used. For a fair comparison, we use the same hyperparameters for the original ResNet-50 \cite{he2015deep} and our model. The quadratic ResNet model is obtained by removing six layers in the fourth stage of the original ResNet-50 with a $3$-layer QNN. 

\begin{table}[! hbt]
    \centering
       \caption{Classification accuracy on the ImageNet.}
    \begin{tabular}{lccc}
    \hline 
    Model &  \#Params & Top 1 (\%) & Top 5 (\%)\\
    \hline
    ResNet-50 & 25M & 75.3 & 92.4 \\
    
    QResNet-47 & 24M & \textbf{75.4} & \textbf{92.5} \\
    \hline
    \end{tabular}
    \label{tab:imagenet}
\end{table}

The comparative result is reported in Table~\ref{tab:imagenet}, \ff{where the performance of ResNet-50 was reported officially in mmclassification package \cite{2020mmclassification}}.  It is seen that a quadratic network with fewer parameters can achieve slightly better top-1 and top-5 accuracy than the conventional network. Since both the improvement and parameter saving are only moderate, the parametric efficiency of quadratic networks in this experiment is also moderate. We argue that the parametric efficiency of quadratic networks holds in general, but how much parametric efficiency can be gained actually depends on the task.

\section{Real-world Applications}

\subsection{Efficiency on Bearing Fault Diagnosis}

Ensuring the reliability of rotating machines such as wind turbines
and aircraft engines is a critical issue in industrial fields, which can avoid tremendous economic loss. According to the statistics \cite{bonnett2008increased}, bearing faults account for as much as 70\% electromagnetic drive system failures. Therefore, it is of great importance to identify the sources of bearing faults. Here, we use a quadratic convolutional neural network (QCNN) to solve the bearing fault diagnosis problem. Such a problem is reduced to classifying faults into different categories. Unlike image classification, bearing fault diagnosis models process temporal vibration signals measured by acceleration sensors. Thus, these models are 1D convolutional neural networks. 

\begin{figure}[htbp]
    \centering
    \includegraphics[width=\linewidth]{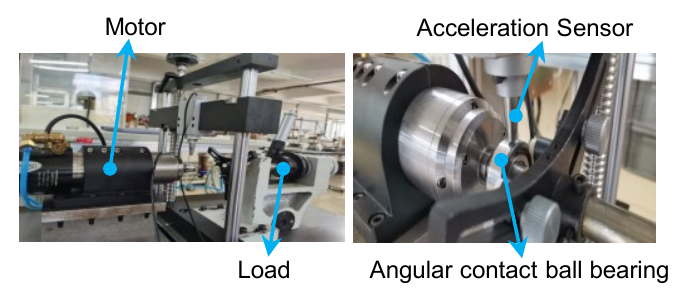}
    \caption{The test rig of angular contact ball bearings.}
    \label{fig:HITdataset}
\end{figure}

\textbf{Dataset preparation.} Our collaborators have collected the angular contact ball bearing dataset. The selected bearing is HC7003, which is a common high-speed spindle bearing. The bearing signal was collected in MIIT Key Laboratory of Aerospace Bearing Technology and Equipment, Harbin Institute of Technology. The fault was injected at the outer race (OR), inner race (IR), and ball of the bearing using laser engraving with 1-3 levels (minor, moderate, severe). Bearing faults are cracks of the same size but different depths. The deeper, the more severe the fault is. As a result, bearing signals are attributed to ten classes \{"healthy", "ball cracking (minor)", "ball cracking (moderate)", "ball cracking (severe)", "outer race cracking (minor)", "outer race cracking (moderate)", "outer race cracking (severe)", "inner race (minor)", "inner race (moderate)", "inner race (severe)"\}. Figure \ref{fig:HITdataset} shows the test rig. In the test, the constant motor speed (1800 r/min) was set, and NI USB-6002 was used to acquire vibration signals with the 12kHz sampling rate. We record 47s of bearing vibration (561,152 points per category). Figure \ref{fig:raw} shows the raw signals in the time domain with respect to the ten classes of our dataset.


\begin{figure}[h]
    \centering
    \includegraphics{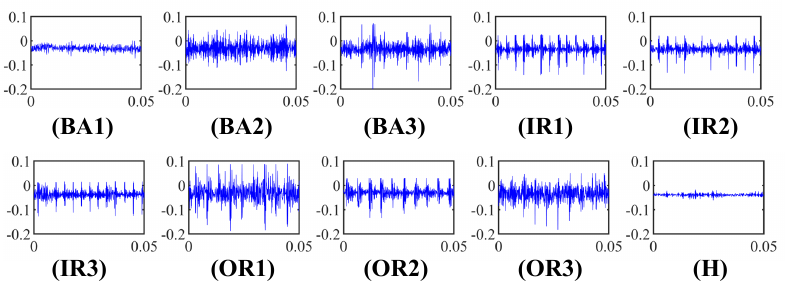}
    \caption{{Raw signals with respect to ten classes of our dataset.}}
    \label{fig:raw}
\end{figure}

\textbf{Experimental setups.} \ff{We use the SOTA 1D-CNN models for comparison, \textit{e.g.}, DCA-BiGRU \cite{zhang2022fault}, AResNet \cite{zhongdevelopment}, RNN-WDCNN \cite{shenfield2020novel}, MA1DCNN \cite{wang2019understanding}, and WDCNN \cite{zhang2017new}.} All these models except WDCNN incorporate extra modules to the CNN backbone, \textit{e.g.}, attention module (AResNet, MA1DCNN) and recurrent convolutional network (RNN-WDCNN). The QCNN shares the same structure as WDCNN but introduces quadratic neurons in convolutional layers. For a full exploitation of data, each signal is sliced into 2,048 points as the input. The data set is split into the training, validation, and test sets with a ratio of 0.5: 0.25: 0.25. The hyperparameters of all methods are specified as follows: the learning rate is 0.3; the batch size is 64; the number of training epochs is 200. Moreover, we adopt Adam as an optimizer. For QCNN, the learning rate of quadratic terms in the ReLinear strategy is set to 0.1. All methods are executed ten times to compute the mean and variance, and experiments are performed on an RTX 3080 Ti 12GB GPU. 

\begin{table}[htbp]
\centering
\caption{The sizes of different models and models' average accuracy ($\%$) on the angular contact ball bearing dataset. Where GRU is a gate recurrent unit, and RNN is a parallel recurrent neural network. }
\scalebox{0.82}{
\begin{tabular}{l|cccc}
\hline
Model     & \# Channel        & \#Params & \#FLOPs  & Accuracy(\%)                       \\ \hline
DCA-BiGRU & [4,2,3]+2*GRU                  & 150.02K  & 1220.01M & 97.42$\pm$1.39                     \\
AResNet   & [3*128,256,2*512]                  & 2650.11K & 260.40M  & 96.14$\pm$2.10                      \\
RNN-WDCNN & {[}2*16,2*32,6*64{]}+RNN                  & 554.10K  & 54.30M   & 91.48$\pm$3.32                      \\
MA1DCNN   & [2*32,2*64,2*128]                  & 324.68K  & 299.70M  & 93.12$\pm$3.10                      \\ 
WDCNN     & {[}16,32,4*64{]}   & 66.79K   & 1.61M    & 97.60$\pm$1.46                      \\ 
QCNN    & {[}6*16{]}         & 16.50K   & 1.39M    & \textbf{98.80$\pm$0.85}               \\
 \hline
\end{tabular}}
\label{tab:bearing_results}
\end{table}

\textbf{Experimental results.} The sizes of different models and models' performance are listed in Table \ref{tab:bearing_results}. First, due to the employment of extra modules such as attention module and RNNs, DCA-BiGRU, AResNet, RNN-WDCNN, and MA1DCNN have a large number of parameters. But they seem not to boost the performance significantly. Second, QCNN achieves the highest accuracy and the lowest variance among all models. As opposed to WDCNN, QCNN performs better, despite that QCNN has 5.5 times fewer parameters and 0.18 times fewer FLOPs. Relative to other state-of-the-art competitors, QCNN is much more compact and also uses much fewer FLOPs. Therefore, we summarize that QCNN enjoys high parametric efficiency in the bearing fault diagnosis task, \ff{and the superior performance of QCNN is not due to the increased parameters.}


\begin{figure*}[htbp]
    \centering
    \includegraphics[width=\linewidth]{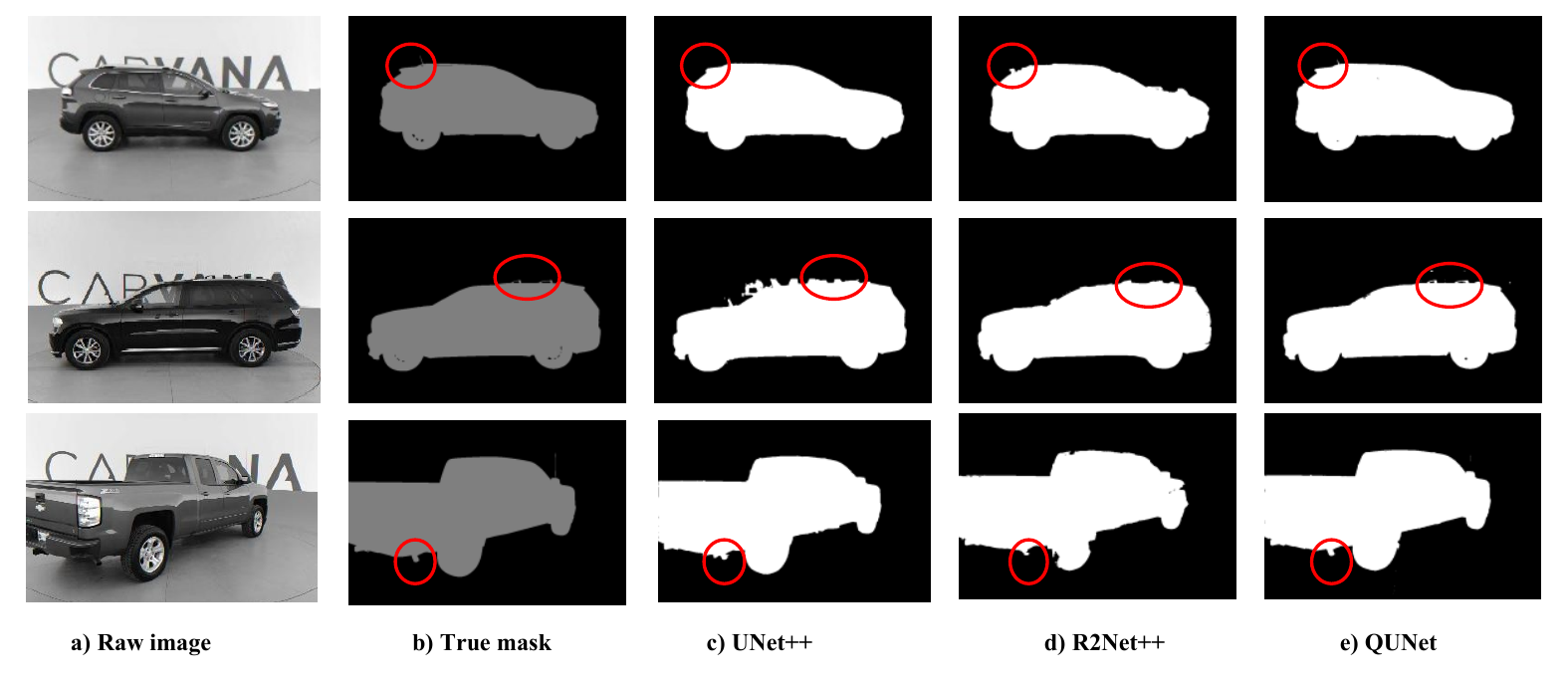}
    \caption{Car segmentation results of three compared methods.}
    \label{fig:car}
\end{figure*}

\subsection{Efficiency on the Car Dataset}


Car dealers need high-resolution, standardized, multi-angle car photos to showcase their products to customers. However, the sophisticated photographic environment often makes it challenging to obtain the ideal photo in a single shot. This necessitates professional image editors to post-process photos for better image quality. To do so, automatic and accurate car segmentation is an important intermediate step for the sake of reducing manual labor and boosting efficiency. A good car segmentation algorithm can substantially improve the quality of the final product and finally expedite sales. Here, we evaluate if a quadratic network has the parametric efficiency on the car segmentation task.

\textbf{Experimental setups.} The dataset in this task was obtained from the Carvana Image Masking Challenge \cite{carvana}. For convenience, we call this dataset the Carvana dataset. Carvana, an online car company, collected images of 318 unique vehicles, each with 16 standard rotating images of a high resolution 1,918 $\times$ 1,080, but the bright reflection and similar color backgrounds make this task tricky. 


\ff{The challenge was completed in 2018, and the scheme of the champion team used a basic UNet with a VGG backbone. Therefore, we develop QUNet models by replacing neurons in UNet with quadratic neurons, and compare it with SOTA models in this task: UNet++ \cite{zhou2018unet++} and R2UNet \cite{alom2018recurrent}.} Because the original image is too large, we follow the preprocessing method of some achieved top-level teams that resize each image into $960 \times 640$. The data set is split into training, validation, and test sets with a ratio of 0.8: 0.1: 0.1. The hyperparameters of all methods are the same for a fair comparison. Specifically, the learning rate is $1 \times 10^{-5}$ and follows a cosine annealing schedule \cite{loshchilov2016sgdr}; the batch size is 2; the number of training epochs is 5. In addition, we use RMSProp as an optimizer to train baseline methods and Adam for QUNet. In particular, for QUNet, we adopt the ReLinear strategy, where the learning rate of quadratic terms is set to $10^{-4}$. We also employ the gradient clip norm method with a maximum norm value of 0.01 to constrain the over-growth of weights.

\begin{table}[htbp]
\centering
\caption{The sizes and FLOPs of different models and models' segmentation performance on the Carvana dataset.}
\scalebox{0.83}{
\begin{tabular}{l|ccccc}
\hline 
Method  & UNet Blocks           & \#Params & \#FLOPs & mACC(\%) & mIOU(\%) \\ \hline 
UNet++& [32,64,128,256,512]   & 9.16M   & 0.33T  & 99.01& 95.62  \\
R2UNet  & [64,128,256,512,1024] & 39.09M  & 1.43T  & 99.08   & 95.41   \\
QUNet & [32,64,128,256]       & 5.78M   & 0.26T  & 99.59   & 98.08   \\ \hline
\end{tabular}}
\label{tab:car}
\end{table}


\textbf{Experimental results.} Table \ref{tab:car} summarizes the sizes and FLOPs of different models and models' segmentation performance on the Carvana dataset. The performance metrics are the mean accuracy (mACC) and mean intersection over union (mIOU). First, all baselines perform satisfactorily because this task is not extremely difficult. There is only a single target in the image, and the distinction between a car and the background is salient.
Second, our QUNet reduces the number of parameters and FLOPs by cutting the number of blocks in all stages. Notwithstanding, QUNet still outperforms its competitors, \textit{i.e.,} the gain in mIOU is even over 2.4\%. At last, as visualized in Figure \ref{fig:car}, UNet++ fails to reconstruct the entire edges, while R2UNet can not precisely retain the edges either, \textit{e.g.}, the car's profile is distorted. In contrast, the segmentation of QUNet is close to the authentic mask. \ff{Since quadratic networks have parametric efficiency in handling the car segmentation task, the superiority of quadratic networks in this task is due to the intrinsic nonlinear representation instead of increased parameters.}

\subsection{Efficiency on the Cell Dataset}

Cell segmentation is a crucial task in the field of microscopic image analysis, pertaining to the identification and localization of the boundaries of cells within an image \cite{Rohan2021}. A well-performed cell segmentation algorithm can facilitate numerous downstream biomedical applications, and is a cornerstone technology of cellular image-based analysis \cite{lux2020cell}. Recent advances in deep learning models, specifically the UNet architecture, have given rise to huge progress in this field. Here, we evaluate if a quadratic network has the parametric efficiency on the car segmentation task.

\textbf{Experimental setups.} The Data Science Bowl 2018 (DSB2018) dataset contains 670 segmented nuclei images acquired under a variety of conditions. The goal is to find the nuclei in images to assist \textit{post hoc} medical analysis. These nuclei come from various cell types \cite{DSB2018}. As opposed to the car segmentation in the Carvana dataset, the image of DSB2018 contains a large number of small cell nuclei, which is a typical small target segmentation problem and
requires a model to delicately tell each nucleus. All images are compressed to the size of $256 \times 256$ such that they can be fed directly to the neural network. 

In this experiment, we choose ResNet34UNet \cite{ronneberger2015u}, UNet++ \cite{zhou2018unet++}, ChannelUNet \cite{chen2019channel}, AttentionUNet, R2UNet \cite{alom2018recurrent}, and FCN \cite{sun2018fully} as baselines and accordingly design a QUNet by  replacing neurons in UNet with quadratic neurons. The DSB2018 dataset is split into the training, validation, and test sets with a ratio of 0.8: 0.1: 0.1. The hyperparameters of all methods are the same for a fair comparison. The learning rate, the batch size, and the number of epochs are set to $10^{-3}$, 8, and 20, respectively. We use Adam as an optimizer for training all methods. In particular, for QUNet, we adopt the ReLinear strategy as \cite{fan2021expressivity}, where the learning rate of quadratic terms is set to $1\times10^{-4}$. We also employ the gradient clip norm method with a maximum norm value of 0.01 to constrain the over-growth of weights. 

\textbf{Experimental results.} We summarize the segmentation results and models' parameters and FLOPs in Table \ref{tab:cell}. First, QUNet delivers the superior performance compared to its counterparts. It has the highest mACC and mIOU and the fewest parameters, which strongly confirms the parametric efficiency. Second, to visually appreciate the performance of different algorithms, the segmentation masks are shown in Figure \ref{fig:cell}. In the larger circle, a superfluous mask appears. ChannelUNet and AttentionUNet misjudge the spots and mark them as a cell. R2UNet fails to segment entire cells. In contrast, QUnet segments this group of cells well.
Furthermore, in the smaller circle, two cells are so close that masks of Res34UNet, UNet++, ChannelUNet, AttentionUNet, and FCN32s appear to be aliased. Only R2UNet and QUNet exhibit distinguishable boundaries. We think that quadratic neurons have smoother edges when approximating circular regions, therefore, quadratic networks are more natural in segmenting circular features such as cells. \ff{This is a result of intrinsic representation instead of stacking parameters.}

\begin{table}[htbp]
\centering
\caption{Segmentation results and models' parameters and FLOPs on DSB2018 dataset.}
\scalebox{0.80}{
\begin{tabular}{l|ccccc}
\hline
Model        & UNet Blocks           & \#Params &\#FLOPs & mACC(\%) & mIOU(\%) \\ \hline
ResNet34UNet     & [64,128,256,512]      & 21.66M  & 6.11G  & 96.68   & 80.23   \\
UNet++        & [32,64,128,256,512]   & 9.16M   & 34.76G & 96.47   & 80.36   \\
ChannelUNet   & [64,128,256,512,1024] & 49.15M  & 76.32G & 96.84   & 82.66   \\
AttentionUNet & [64,128,256,512,1024] & 34.88M  & 66.68G & 96.90    & 81.72   \\
R2UNet        & [64,128,256,512,1024] & 39.09M  & 0.15T  & 94.74   & 72.17   \\
FCN32s        & [32,64,128,256,512]   & 18.64M  & 22.5G  & 96.39   & 79.37   \\
QUNet         & [32,64,128,256]       & 5.78M   & 27.73G & 97.15   & 82.16   \\ \hline
\end{tabular}}
\label{tab:cell}
\end{table}

\begin{figure}
    \centering
    \includegraphics{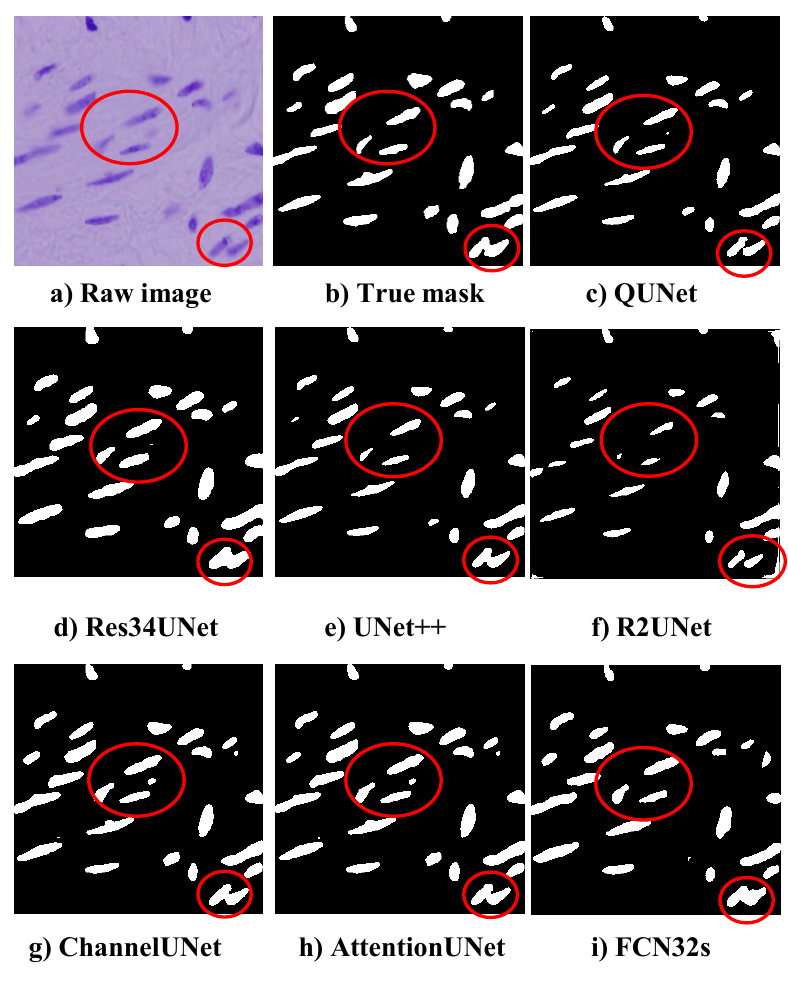}
    \caption{Cells segmentation image of seven compared methods.}
    \label{fig:cell}
\end{figure}

\subsection{Efficiency on the Agriculture Image Dataset}
\textbf{Experimental setups.} The CVPPP Leaf Segmentation dataset \cite{scharr2014annotated} contains well-annotated leaf images, which is used to benchmark the quality of image segmentation tasks in the planting field. In detail, two datasets show different genotypes of Arabidopsis, while another shows tobacco (Nicoticana tobacum) under different treatments. During the training process, the input images are randomly resized in the range of X and X and then cropped into the size of X by X before being fed into the networks. During the inference time, the input images are directly fed into the networks without any augmentation pipeline as the training process. Regarding the training/test split,  A1 and A4 folders are combined, which contain (512 + 2497) images, are used for training, while A2 and A3 (125 + 109) images are for testing. \ff{All
hyperparameters for training are the same as Unet \cite{ronneberger2015u}, PSPNet \cite{zhao2017pyramid}, and
Unet++\cite{zhou2018unet++}}.

\textbf{Experiments results.} We summarize the segmentation results and models' sizes in Table \ref{tab: leafSegmentation}. When integrated with quadratic neurons, the performance on the evaluation metrics like mIOU, F1, Recall, and Precision can be improved by a large margin. Figure \ref{fig: leaf1} visualizes segmentations of different models. It can be seen from Figure \ref{fig: leaf1} that the segmentation maps of quadratic models improve near the stems of the leaves.

\begin{figure}[! bt]
    \centering
    \includegraphics[width=8cm]{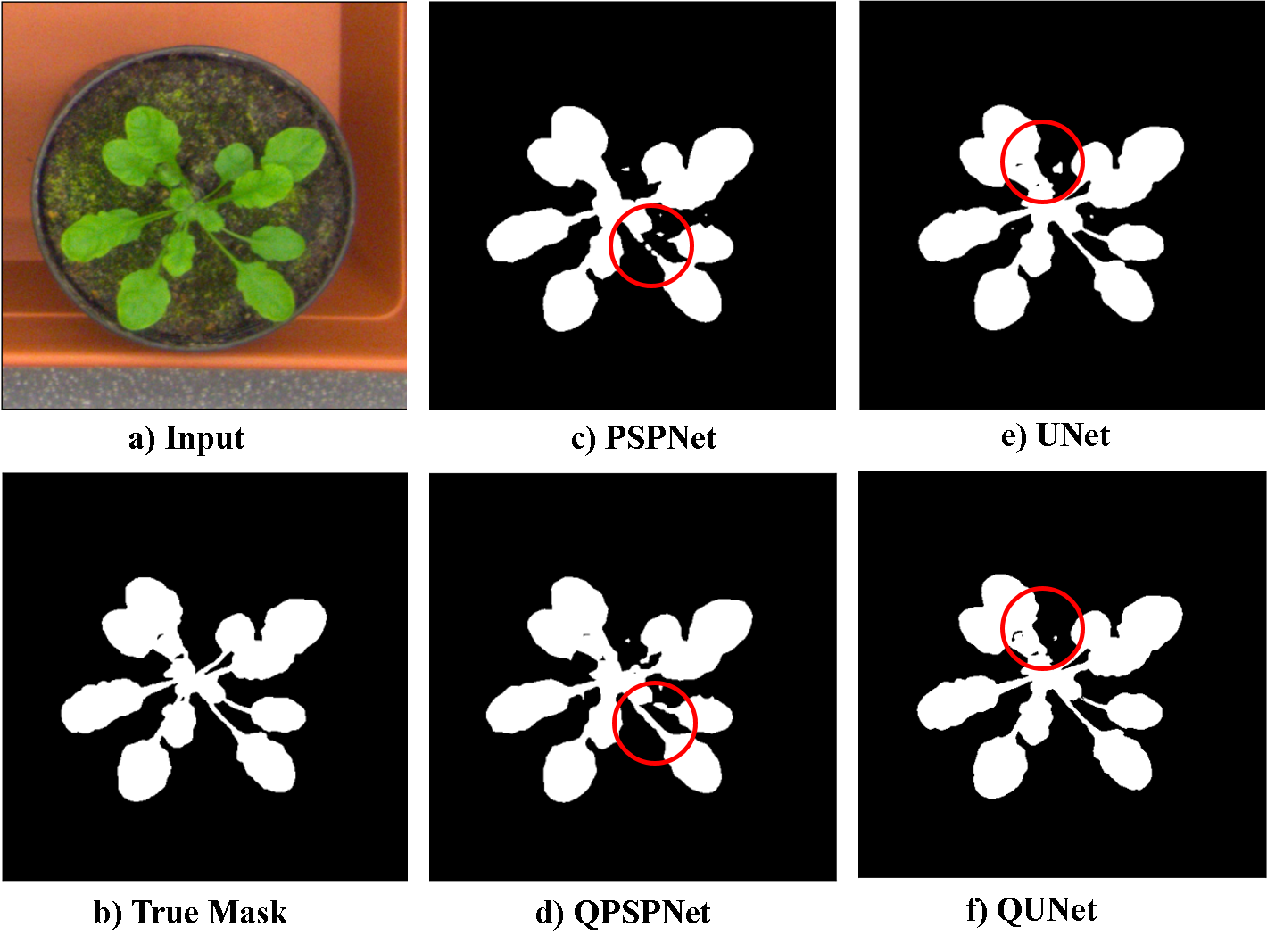}
    \caption{Leaf segmentation image of four compared methods.}
    \label{fig: leaf1}
\end{figure}

\begin{table}[! bt]
    \centering
    \caption{Leaf segmentation}
    \scalebox{0.80}{
    \begin{tabular}{l|cccccc}
    \hline
    Model & \#Params & mIOU & F1 & Recall & Accuracy & Precision \\
    \hline 
    Unet & 25.81M & 80.58$\%$ & 88.13$\%$ & 94.45$\%$ & 99.32$\%$ & 84.76$\%$\\
    QUnet & 24.45M & 83.40$\%$ & 90.08$\%$ & 89.65$\%$ & 99.48$\%$ & 92.72$\%$\\
    \hline 
    PSPNet & 21.44M & 76.54$\%$ & 85.43$\%$ & 94.31$\%$ & 99.27$\%$ & 81.06$\%$\\
    QPSPNet & 21.40M &  82.11$\%$ & 89.18 $\%$ & 89.44$\%$ & 99.50$\%$ & 89.88$\%$\\
    \hline 
    Unet++ & 27.73M & 77.32$\%$ & 86.10$\%$ & 93.56$\%$ & 99.22$\%$ & 83.28$\%$\\
    QUnet++ & 26.09M & 83.27$\%$ & 90.34$\%$ & 93.48$\%$ & 99.35$\%$ & 88.87$\%$\\
    \hline 
    \end{tabular}
    }
    \label{tab: leafSegmentation}
\end{table}

\section{Discussion and Conclusion}

We summarize the parametric efficiency of quadratic networks on different experiments in Table \ref{tab:experiments_summary}. The experiments conducted in various tasks suggest that quadratic networks possess parametric efficiency, but like the conventional, the quadratic network is also not a panacea. When dealing with circular data such as concentric hyperspheres, cells, and leaves, quadratic networks demonstrate higher parametric efficiency compared to conventional counterparts. We believe that these types of tasks are the ideal battlefield and killer mission for quadratic networks to realize their full potential. However, the parametric efficiency is moderate when applied to ImageNet. Since no single neuron type is a universal solution, we believe that introducing neuronal diversity in deep learning is promising. The selection of neurons should take into account the nature of the task. 

\begin{table}[htbp]
\centering
\caption{Tasks and the corresponding achieved parametric efficiency.}
\scalebox{0.9}{
\begin{tabular}{l|l}
\hline
Tasks           & Parametric efficiency         \\ \hline
Concentric hyperspheres & $\bigstar$ $\bigstar$ $\bigstar$ $\bigstar$ $\bigstar$             \\ 
Gaussian mixture data & $\bigstar$ $\bigstar$ $\bigstar$ $\bigstar$               \\    
ImageNet & $\bigstar$               \\ 
Point cloud segmentation & $\bigstar$ $\bigstar$ $\bigstar$ \\ 
Bearing fault diagnosis & $\bigstar$ $\bigstar$ $\bigstar$ $\bigstar$ \\ 
Cell data & $\bigstar$ $\bigstar$ $\bigstar$ $\bigstar$               \\ 
Agricultural data &  $\bigstar$ $\bigstar$ $\bigstar$ $\bigstar$ $\bigstar$               \\ \hline 
\end{tabular}}
\label{tab:experiments_summary}
\end{table}

\ff{In this article, we have adequately verified the efficiency of quadratic networks both theoretically and experimentally, thereby demonstrating that the superior performance
of quadratic networks is not due to the increased parameters
but due to the intrinsic expressive capability.} We have derived quadratic networks to approximate functions on real space and manifolds, thereby proving that a ReLU quadratic network can approximate a class of functions lying in the unit ball of the Sobolev spaces more efficiently than a ReLU conventional one. Similar results have been also extended to a manifold. Also, we have shown the parametric efficiency of quadratic networks from the perspective of Barron space. Furthermore, we have empirically verified the parametric efficiency of quadratic networks through synthetic experiments, benchmarks, and real-world applications. It has been noticed that quadratic networks are good at circular features. In the future, more research should focus on designing suitable neuron types for specific tasks to maximize a model's performance, as well as the computational efficiency of quadratic networks.


\ifCLASSOPTIONcaptionsoff
  \newpage
\fi



%

\bibliographystyle{ieeetr}
\bibliography{bio.bib}

%








\end{document}